\documentclass[10pt]{article}

\usepackage{scicite}

\usepackage{times}

\usepackage{amsthm} 
\usepackage{graphicx}
\usepackage{multicol}
\usepackage{subfigure}
\usepackage{amsmath}
\usepackage{amssymb}
\usepackage{amsfonts}
\usepackage{graphicx}
\usepackage{url}
\usepackage{ccaption}
\usepackage{booktabs}  

\DeclareMathOperator{\R}{\mathbb{R}}

\providecommand{\scal}[2]{\left\langle{#1},{#2}\right\rangle}
\newcommand{\be}{\begin{equation}}
\newcommand{\ee}{\end{equation}}
\newcommand{\bt}{\begin{theorem}}
\newcommand{\et}{\end{theorem}}
\newcommand{\bd}{\begin{definition}}
\newcommand{\ed}{\end{definition}}

\newcommand{\br}{\begin{remark}}
\newcommand{\er}{\end{remark}}

\newtheorem{theorem}{Theorem}
\newtheorem{definition}{Definition}
\newtheorem{conjecture}{Conjecture}

\newtheorem{lemma}{Lemma}
\newtheorem{proposition}{Proposition}
\newtheorem{remark}{Remark}

\newtheorem{corollary}{Corollary}

\newcommand{\Col}{\mathrm{Col}}
\newcommand{\Null}{\mathrm{Null}}
\newcommand{\vect}{\mathrm{vec}}

\newcounter{lastnote}

\title{Theory of Deep  Learning III: 
explaining the non-overfitting puzzle}

\author {Tomaso Poggio$^{\dagger, \ast}$, Kenji  Kawaguchi
  $^{\dagger \top}$, Qianli Liao$^{\dagger}$,
  Brando Miranda$^\dagger$,  Lorenzo Rosasco$^\dagger$\\
{\it with}\\Xavier Boix$^\dagger$,  Jack Hidary$^{\dagger \dagger}$,
Hrushikesh Mhaskar$^{\diamond}$,\\
\\
  \author
  \normalsize{$^\dagger$Center for Brains, Minds and Machines, MIT}\\
  \normalsize{$^{\dagger \top}$CSAIL, MIT}\\
  \normalsize{$^{\dagger \dagger}$GoogleX}\\
  \normalsize{$^{\top \top}$IIT}\\
  \normalsize{$^{\diamond}$Claremont Graduate
University}\\
 \normalsize{$^\ast$To whom correspondence should be addressed; E-mail:
  tp@ai.mit.edu} 
}

\date{}

\usepackage[margin= 0.8in,top=12mm]{geometry}

\newcommand*{\titleAT}{\begingroup
  \newlength{\drop}
  \drop=0.05\textheight
  \begin{center}
  \includegraphics[scale=0.4]{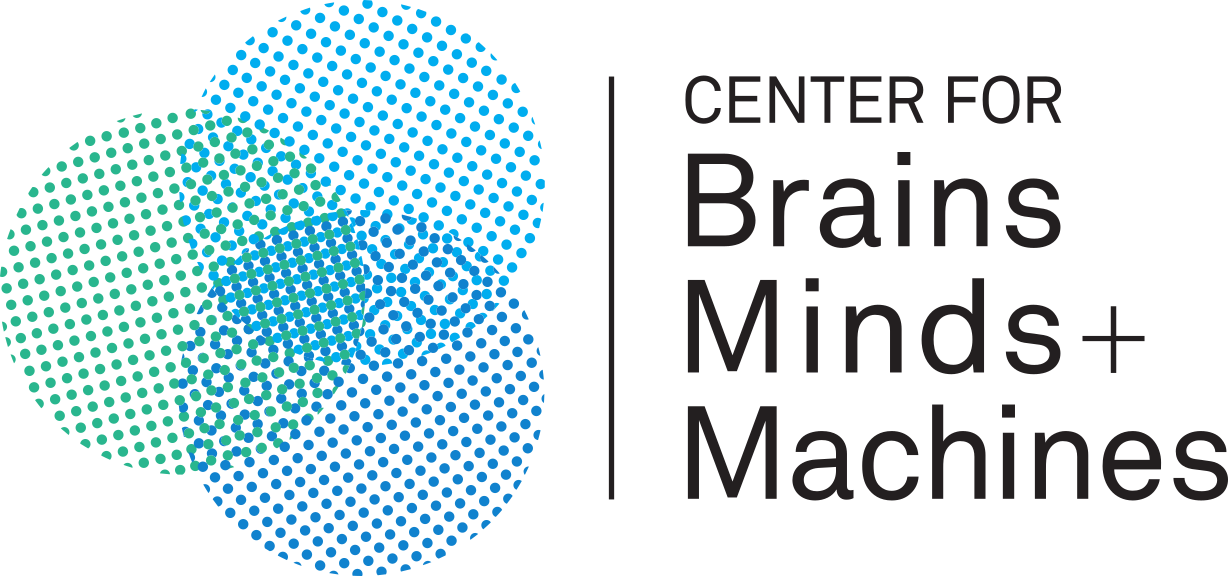} 
  \end{center} 
  \vspace{2pt}\vspace{-\baselineskip}

  \vspace{\drop}
  \textbf{\large{CBMM Memo No. \memonumber}}   \hfill    \textbf{\large{\memodate}} 

  \begin{center}
    \textbf{\huge{\memotitle}}\\
   \textbf{{by}}\\
    \vspace{0.4\drop}
    \large{\memoauthors}
  \end{center}
  \textbf{\large{\noindent Abstract}:} {\memoabstract}

  \begin{minipage}{.15\linewidth}
    \includegraphics[scale=0.1]{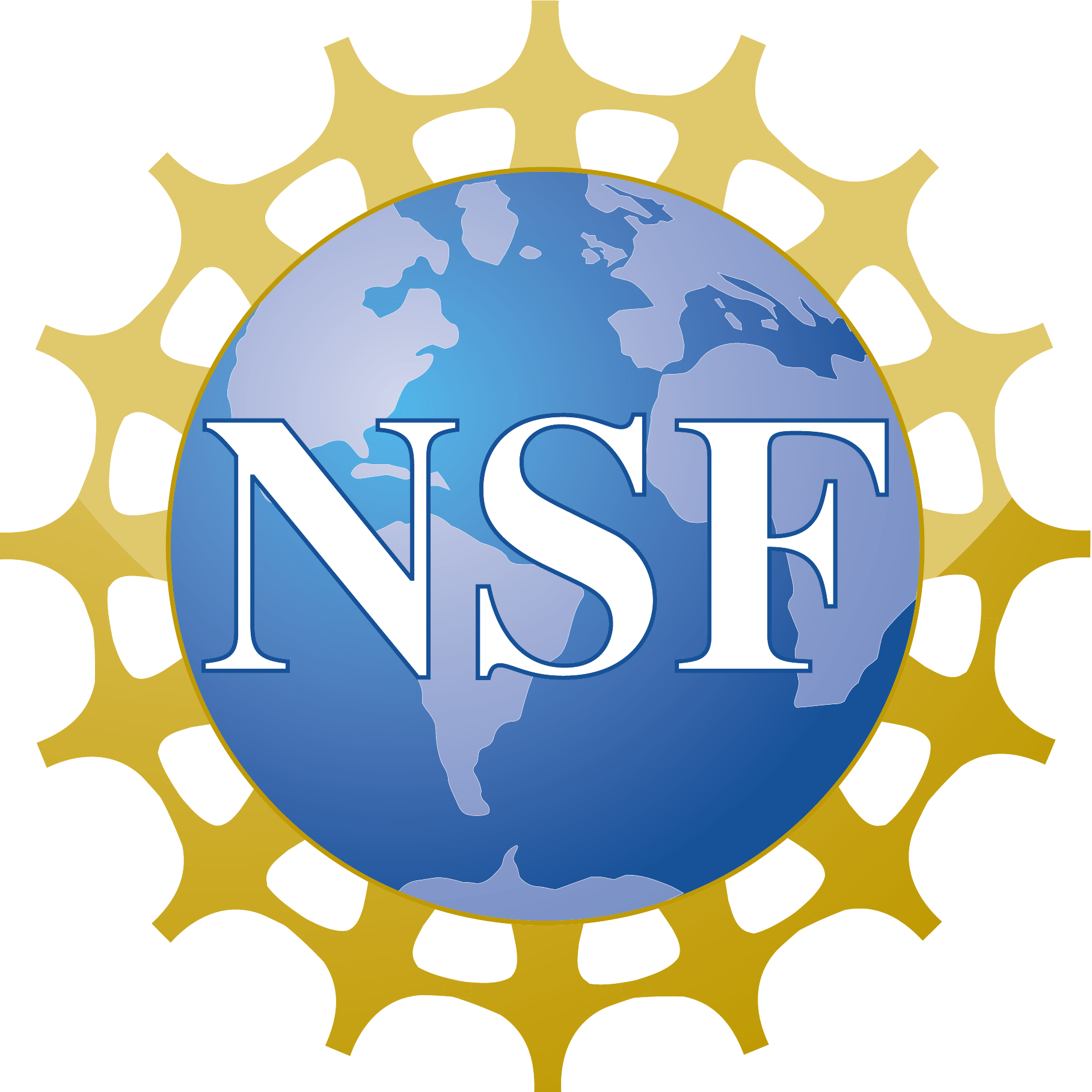}
  \end{minipage}
  \begin{minipage}{.84\linewidth}
    \textbf{{This work was supported by the Center for Brains, Minds and Machines (CBMM), funded by NSF STC award  CCF - 1231216.}}
  \end{minipage}
  \endgroup}

\begin{document}

\def\memonumber{073}
\def\memodate{\today}
\def\memotitle {Theory of Deep  Learning III:  explaining the non-overfitting puzzle}
\def\memoauthors{
  T. Poggio$^{\dagger}$, K. Kawaguchi $^{\dagger \top}$, Q. Liao$^{\dagger}$,  B. Miranda$^\dagger$,  L. Rosasco$^\dagger$\\   
   {\it with} \\
   X. Boix$^\dagger$,  J. Hidary$^{\dagger \dagger}$,  H. Mhaskar$^{\diamond}$,\\
  \small{$^\dagger$Center for Brains, Minds and Machines, MIT}\\
  \small{$^{\top}$CSAIL, MIT}\\
  \small{$^{\dagger \dagger}$Alphabet (Google) X}\\
  \small{$^{\diamond}$Claremont Graduate University}\\  

}

\normalsize \def\memoabstract{ A main puzzle of deep networks revolves
  around the absence of overfitting despite large overparametrization
  and despite the large capacity demonstrated by zero training error
  on randomly labeled data.  In this note, we show that the dynamics
  associated to gradient descent minimization of nonlinear networks is
  topologically equivalent, near the asymptotically stable minima of
  the empirical error, to  linear gradient system in a quadratic
  potential with a degenerate (for square loss)  or almost degenerate
  (for logistic or crossentropy loss)  Hessian. The
  proposition depends on the qualitative theory of dynamical
  systems and is supported by numerical results.

  Our main propositions extend to deep nonlinear networks two
  properties of gradient descent for linear networks, that have been
  recently established \cite{2017arXiv171010345S} to be key to their
  generalization properties:

\begin{enumerate}  
\item Gradient descent enforces a form of
  implicit regularization controlled by the number of iterations, and
  asymptotically converges to the minimum norm solution for
  appropriate initial conditions of gradient descent. This implies
  that there is usually an optimum early stopping that avoids
  overfitting of the loss. This property, valid for the square loss
  and many other loss functions,  is relevant especially for
  regression.
\item For classification, the asymptotic convergence to the minimum
  norm solution implies convergence to the maximum margin solution
  which guarantees good classification error for ``low noise''
  datasets. This property holds for loss functions such as the
  logistic and cross-entropy loss independently of the initial
  conditions.
\end{enumerate}  
The robustness to overparametrization has suggestive implications for
the robustness of the architecture of deep convolutional networks with
respect to the curse of dimensionality.}

\titleAT

\newpage     










\section{Introduction}

In the last few years, deep learning has been tremendously successful
in many important applications of machine learning. However, our
theoretical understanding of deep learning, and thus the ability of
developing principled improvements, has lagged behind. A satisfactory
theoretical characterization of deep learning is finally emerging. It
covers the following questions: 1) {\it representation power} --- what
types of functions can deep neural networks (DNNs) represent and under
which conditions can they be more powerful than shallow models; 2)
{\it optimization} of the empirical loss --- can we characterize the
minima obtained by stochastic gradient descent (SGD) on the non-convex
empirical loss encountered in deep learning? 3) {\it generalization}
--- why do the deep learning models, despite being highly
over-parameterized, still predict well? Whereas there are satisfactory
answers to the the first two questions (see for reviews
\cite{Theory_I,Theory_II} and references therein), the third question
is still triggering a number of papers (see among others
\cite{Hardt2016, NeyshaburSrebro2017, Sapiro2017, 2017arXiv170608498B,
  Musings2017}) with a disparate set of partial answers.  In
particular, a recent paper titled ``Understanding deep learning
requires rethinking
generalization''\cite{DBLP:journals/corr/ZhangBHRV16} claims that the
predictive properties of deep networks require a new approach to
learning theory. This paper is motivated by observations and
experiments by several authors, in part described in
\cite{Musings2017}, where a more complex set of arguments is used to
reach similar conclusions to this paper. It shows that the
generalization properties of linear networks described in
\cite{2017arXiv171010345S} and \cite{RosascoRecht2017} can be extended to deep
networks.

Using the classical theory of  ordinary differential equations, our approach replaces a
potentially fundamental puzzle about generalization in deep learning
with elementary properties of gradient optimization techniques. This
seems to explain away the generalization puzzle of today's deep
networks and to imply that there is no fundamental problem with classical
learning theory. The paper is at the level of formal rigor of a
physicist (not a mathematician).

In this paper we focus on gradient descent (GD) rather than stochastic
gradient descent (SGD). The main reason is simplicity of analysis,
since we expect the relevant results to be valid in both
cases. Furthermore, in ``easy'' problems, such as CIFAR, one can
replace SGD with GD without affecting the empirical results, given
enough computational resources. In more difficult problems, SGD, as
explained in \cite{Theory_IIb}, not only converges faster but also is
better at selecting global minima vs. local minima. 


Notice that in all computer simulations reported in this paper, we
turn off all the ``tricks'' used to improve performance such as data
augmentation, weight decay etc. in order to study the basic properties
of deep networks optimized with the SGD or GD algorithm (we keep
however batch normalization in the CIFAR experiments). We also reduce
in some of the experiments the size of the network or the size of the
training stet. As a consequence, performance is not state of the art,
but optimal performance is not the goal here (in fact we achieve very
close to state-of-the-art performance using standard number of
parameters and data and the usual tricks). In our experiments we used
the cross entropy loss for training, which is better suited for
classification problems (such as for CIFAR), as well as the square
loss for regression (see SI \ref{SquareClass}).

Throughout the paper we  use the fact that that a polynomial
network obtained by replacing each ReLUs of a standard deep
convolutional network with its univariate polynomial approximation
shows all the properties of deep learning networks. As a side remark,
this result shows that RELU activation functions can be replaced by
smooth activations (which are not {\it positively homogeneous})
without affecting the main properties of deep networks.

\section{Overfitting Puzzle}

Classical learning theory characterizes generalization behavior of a
learning system as a function of the number of training examples
$n$. From this point of view deep learning networks behave as
expected: the more training data, the smaller the test error, as shown
in the left of Figure \ref{TwoRegimes}. Other aspects of their
learning curves seem less intuitive but are also easy to explain.
Often the test error decreases for increasing $n$ even when the
training error is zero (see Figure \ref{TwoRegimes} left side). As
noted in \cite{2017arXiv171010345S} and \cite{RosascoRecht2017}, this
is because the classification error is reported, rather than the loss
minimized for training, e.g.  cross-entropy. It is also clear that
deep networks do show {\it generalization}, technically defined as
convergence for $n \to \infty$ of the training error to the expected
error.  The left and the right plots in Figure \ref{TwoRegimes}
indicates generalization for large $n$. This is expected from previous
results such as Bartlett's\cite{AntBartlett2002} and especially from
the stability results of Recht\cite{DBLP:journals/corr/HardtRS15}.

The property of generalization, though important, is however of
academic importance only, since deep networks are typically used in
the overparametrized case, that is, in a regime in which several
classical generalization bounds are not valid. The real puzzle in this regime-- and
the focus of this paper -- is the apparent lack of overfitting,
despite usually large overparametrization, that is many more
parameters than data. The same network which achieves zero training
error for randomly labeled data (right plot in Figure \ref{TwoRegimes}), clearly showing large capacity, does
not have any significant overfitting in classifying normally labeled
data (left plot in Figure \ref{TwoRegimes}). 
 

\begin{figure*}[h!]\centering
\includegraphics[width=1.0\textwidth]{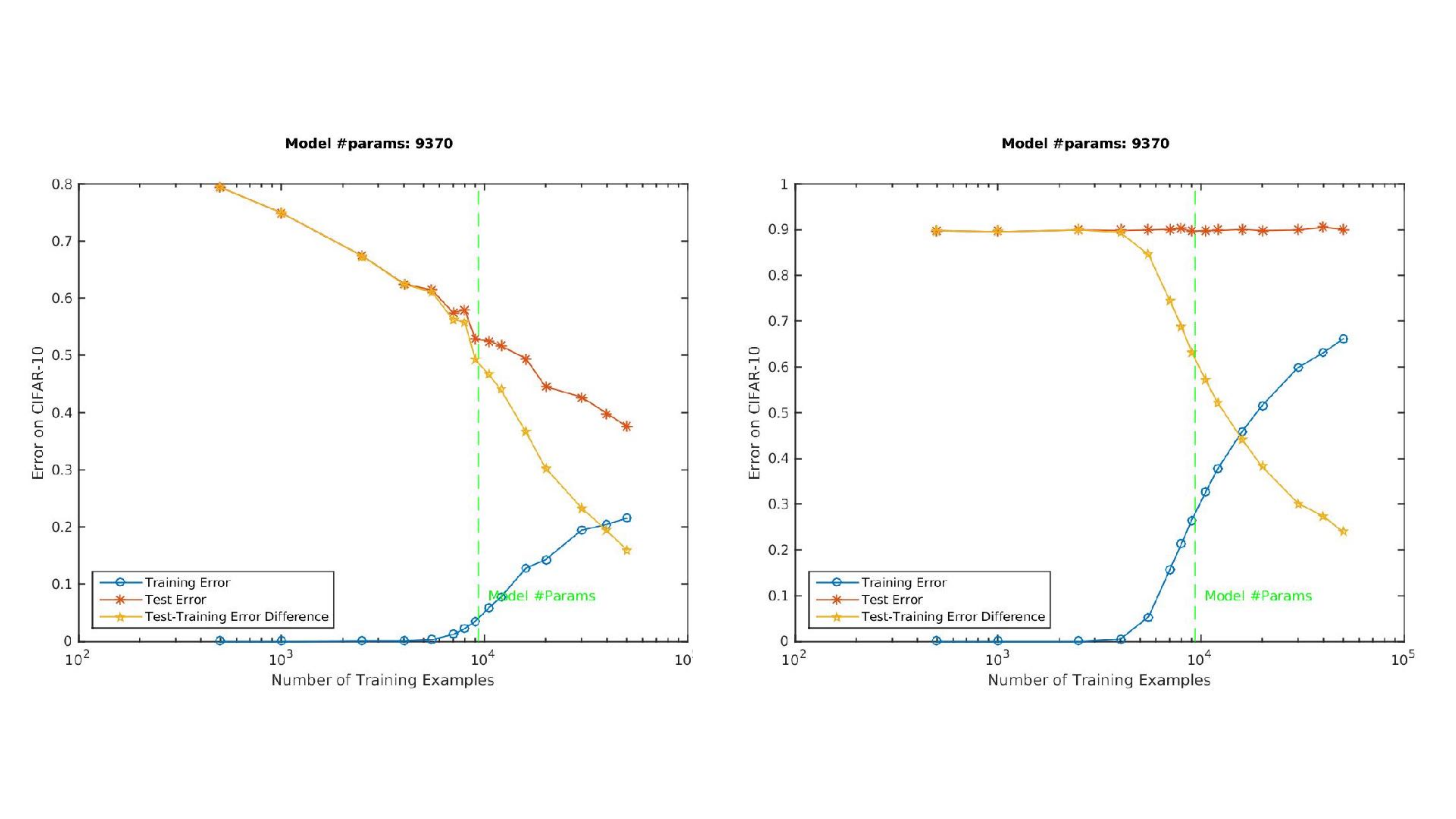}
\caption{\it The figure (left) shows the behavior the classification
  error for a deep network
  with RELU activations trained on subsets of the CIFAR database by
  minimizing the crossentropy loss. The
  figure on the right shows the same network trained on subsets of the
  CIFAR database in which the labels have been randomly scrambled.
  The network is a 5-layer all convolutional network (i.e., no
  pooling) with 16 channels per hidden layer , resulting in only
  $W \approx 10000$ weights instead of the typical $300,000$. Neither
  data augmentation nor regularization is performed.}
\label{TwoRegimes}
\end{figure*}

As a general warning it should be clear that overparametrization
cannot be decided simply in terms of the exact number of parameters vs
number of data. As shown in SI \ref{DeepNetsLinSusubsection}, there is
trivial and nontrivial parametrization. The number of parameters is
just a rough guideline to overparametrization.

\section{Linear Networks}

In the special case of linear models, an explanation for the lack of
overfitting has been recently proposed in \cite{2017arXiv171010345S}
and \cite{RosascoRecht2017}.  Two main properties are suggested to be
important: the difference between classification error and loss, and
the implicit regularization properties of gradient descent methods.
Gradient descent iteratively controls the complexity of the model and
the regularization parameter is the inverse of the number of
iterations. As the number of iterations increases, less regularization
is enforced, and in the limit the minimum norm solution is selected.
The latter is the maximum margin solution ensuring good classification
error for separable problems. Though this is valid for several
different loss functions, the maximum margin solution has somewhat
different properties in the case of square loss vs. logistic or
cross-entropy loss.

\subsection{Gradient descent methods yield implicit regularization}

This section is based on \cite{RosascoRecht2017}. 

In linear models with square loss, it is known that, for appropriate
initial conditions, GD provides {\it implicit regularization}
controlled by the number of iterations $j$ (for a fixed gradient step)
as $\lambda \propto \frac{1}{j}$ (see \cite{rosasco2015learning} and
section \ref{Lorenzo} of SI for a summary). As the number of
iterations increases, the equivalent, implicit $\lambda$ decreases and
a larger set of possible solutions is explored (in analogy to Tikhonov
regularization with a decreasing regularization parameter, see Figure
\ref{CIFARclass_pert}).  For overparametrized models one expects
overfitting in the loss: the asymptotic solution for $t \to \infty$ --
which corresponds to $\lambda \to 0$ -- converges in the limit to the
minimum norm solution (for square loss this is the pseudoinverse),
which is not a regularized solution itself, but the limit of a
regularized solution (see for instance section $5$ in
\cite{DBLP:journals/corr/ZhangBHRV16}).  This overfitting in the loss
can be avoided, even without explicit regularization, by appropriate
early stopping, as shown in Figure \ref{Brando1}.

\subsection{Implicit regularization yields margin maximization}
\label{marginsection}

A direct consequence of the implicit regularization by gradient
descent is that the solution obtained after sufficiently many
iterations is the minimum norm solution.  The proof holds in the case
of linear networks for a variety of loss functions and in particular
for the square loss
(\cite{DBLP:journals/corr/ZhangBHRV16} and SI section \ref{MinNorm}
for a short summary).

Convergence to the minimum norm solution is suggested in
\cite{2017arXiv171010345S} to be the main reason behind the lack of
overfitting in classification. Indeed, the minimum norm solution is
known to maximize classification margin, which in turn ensures good
classification error for ``low noise'' data sets, which, informally,
exactly the data sets where the classes are separated by a nicely
behaving margin.  The precise conditions on the data that imply good
classification accuracy are related to Tsybakov conditions (see for
instance \cite{Yao2007}).

In the case of the logistic and crossentropy loss, properties of the
convergence to the maximum margin solution which is well known for
linear networks with the square loss, should be qualified further,
since proofs such as \cite{DBLP:journals/corr/ZhangBHRV16} need to be
modified.   Lemma 1 in \cite{2017arXiv171010345S} shows that for loss
functions such as cross-entropy, gradient descent on linear networks
with separable data converges {\it asymptotically to the max-margin
  solution with any starting point $w_0$, while the norm $||w||$
  diverges} in agreement with Lemma
\ref{robustness-lemma}. Furthermore, \cite{2017arXiv171010345S} prove
that the {\it convergence to the maximum margin solution is very slow
  and only logarithmic in the convergence of the loss itself}. This
explains why optimization of the logistic loss helps decrease the
classification error in testing, even after the training
classification error is zero and the training  loss is very small, as
in Figure \ref{TwoRegimes}.

{\bf Remarks}

\begin{itemize}
\item An alternative path to prove convergence to a maximum margin
  solution, connecting it directly to {\it flat minima} and assuming a
  classification setting with a loss such as the logistic loss, is
  described in \cite{Musings2017} and in Lemma \ref{robustness-lemma}
  of the SI: {\it SGD maximizes flatness of the minima, which is
    equivalent to robust optimization, which maximizes margin}.

\item The crossentropy loss is an upper bound for the classification error. 

\end{itemize}

\section{Deep networks at global minima are topologically equivalent to linear networks} 

{\it Qualitative properties of the dynamical system}

To be able to use the linear results also for deep nonlinear networks,
we introduce here classical properties of dynamical systems defined in
terms of the gradient of a Lyapunov function such as the training
loss.

The gradient dynamical system corresponding to training a deep network
with gradient descent with a loss function $L$ is
\begin{equation}
\dot{W} = -\nabla_{W} L(W) = - F(W).
\label{GradSys}
\end{equation}

The simplest example is $L(w)=\sum_1^n (f_w(x_i)-y_i)^2$ where $f$ is
the neural network parametrized by the weights and $n$ is the number
of example pairs $x_i, y_i$ used for training. 

We are interested in the qualitative behavior of the dynamical system
near stable equilibrium points $W^*$ where $F(W^*)=0$. One of the key
ideas in stability theory is that the qualitative behavior of an orbit
under perturbations can be analyzed using the linearization of the
system near the orbit.  Thus the first step is to linearize the
system, which means considering the Jacobian of $F$ or equivalently
the Hessian of $L$ at $W^*$, that is

 \begin{equation}
  (\mathbf{H} L)_{ij} = \frac{\partial^{2} L}{\partial
                      w_{i} \partial w_{j} }
 \end{equation}. 

We obtain

\begin{equation}
\dot{W} = - H W,
\label{GradSysLin}
\end{equation}

\noindent where the matrix $H$, which has only real
eigenvalues (since it is symmetric), defines in our case (by
hypothesis we do not consider unstable critical points) two main subspaces:

\begin{itemize}
\item the stable subspace spanned by eigenvectors corresponding to
  negative eigenvalues
\item the center subspace corresponding to zero eigenvalues.
\end{itemize}

The center manifold existence theorem \cite{Carr81} states that if $F$ has $r$
derivatives (as in the case of deep polynomial networks) then at every
equilibrium $W^*$ there is a $\mathbf{C}^r$ stable manifold and a $\mathbf{C}^{r-1}$
center manifold which is sometimes called {\it slow manifold}. The
center manifold emergence theorem says that there is a neighborhood of
$W^*$ such that all solutions from the neighborhood tend exponentially
fast to a solution in the center manifold. In general properties of the
solutions in the center manifold depends on the nonlinear parts of
$F$. We assume that the center manifold is not unstable in our case,
reflecting empirical results in training networks.

{\it The Hessian of deep networks}

%
%

The following two separate results imply that the Hessian of the
loss function of deep networks is indeed degenerate at zero-loss minima
(when they exist):

\begin{enumerate}
\item Polynomial deep networks can approximate arbitrarily well a
  standard Deep Network with ReLU activations, as shown theoretically
  in section \ref{SectionPolynomials} of SI and empirically in Figures
  \ref{TwoRegimesPol} and \ref{GreatPlotPol}. Bezout theorem about the
  number of solutions of sets of polynomial equations suggests many
  {\it degenerate} zero-minimizers of the square loss for polynomial
  activations \cite{Theory_II}.  Note that the energy function $L$ determining
  the gradient dynamics is a polynomial.
\item Section \ref{MultiNets} in SI proves theorem
  \ref{thm:zero_eigenvalue} of which an informal statement is: 
  
\begin{lemma}
\label{Kenji}
Assume that gradient descent of a multilayer overparametrized network
with nonlinear activation (under the square loss) converges to a
minimum with zero error. Then the Hessian at the minimum has one or
more zero eigenvalues.
\end{lemma}
\end{enumerate}

The fact that the Hessian is degenerate at minima (many zero
eigenvalues) is in fact empirically well-established (for a recent
account see \cite{DBLP:journals/corr/SagunBL16}).

The following Proposition follows, where we use the term ``stable
limit point'' or stable equilibrium, to denote an asymptotic limit for $t \to \infty$ of
Equation \ref{GradSys}, which is stable in the sense of Liapunov (e.g.
there is a $\delta$ such that every $W$ within an $\delta$
neighborhood of the limit $W*$ will remain under the gradient dynamics
at distance smaller than $\epsilon$ of it).

\begin{proposition}
\label{proposition}
If 
\begin{enumerate}
 \item $W^*$ is a stable equilibrium of the gradient dynamics of a
   deep network trained with the square loss
\item $W^*$ is a  zero minimizer 
\end{enumerate}
then the ``potential function'' $L$ is locally quadratic and
degenerate.
\end{proposition}

{\it Proof sketch}

\begin{itemize}
\item The first hypothesis  corresponds to zero gradient at $W^*$,
  that is
\begin{equation}
(\nabla_w L) (W^*) = 0;\quad\quad\quad
\end{equation}

\item Lemma \ref{Kenji}  shows that if $L(W^*)=0$ the Hessian has one or more zero eigenvalues, that is 

\begin{itemize}
\item let $\{v_1,\dots,v_d\}$ be an orthonormal basis
of eigenvectors of the Hessian at $W^*$, $\mathbf{H} L(W^*)$, and
let $\lambda_1 \ge \cdots \ge \lambda_d \ge 0$ be the corresponding
eigenvalues (the Hessian is a real symmetric and positive
semidefinite matrix). Then, any vector $w \in \mathbb R^d$ can be
written as $w=\sum_{k}^{d}\alpha_k v_k$ and

\begin{equation}
L(W^*+w) = w^\top \left(\mathbf{H} L(W^*)\right) w= \sum_{k}^j \lambda^2_k\alpha_k^{2},
\label{quadraticapproximation}
\end{equation}


\noindent which is convex in $\alpha_1,\cdots, \alpha_j$ and
degenerate in the directions $\alpha_{j+1}, \cdots, \alpha_d$, for
which $\lambda_{j+1}=\lambda_{d}=0$.
\end{itemize}

Notice that the term ``locally quadratic'' in the statement of the
Proposition means that $L(W^* +w)$ is quadratic in $w$.

If $L(w)$ is a polynomial in $w$, the hypothesis implies that all
linear terms disappear at $w*$ and that the quadratic terms have
dimensionality $d_e<d$ with $d_e > 0$. Notice that $L$, as a
polynomial, cannot have a flat minimum in a finite neighborhood around
$W^*$ without being identically zero (thus the ``isotropically flat''
minima claimed in \cite{Musings2017} do not exist if $L$ is a
polynomial; furthermore the claim \cite{DBLP:journals/corr/DinhPBB17} that flat valleys are not directly
important for generalization is then correct).

\end{itemize}
The informal statement of this proposition is

{\it Each of the zero minima found by GD or SGD is locally well
  approximated by a quadratic degenerate minimum -- the
  multidimensional equivalent of the two-dimensional minimum of Figure
  \ref{degeneratesquareloss}. The dynamics of gradient descent for a
  deep network near such a minimum is topologically equivalent to the
  dynamics of the corresponding linear network. }

\begin{figure*}[h!]\centering
\includegraphics[width=1.0\textwidth]{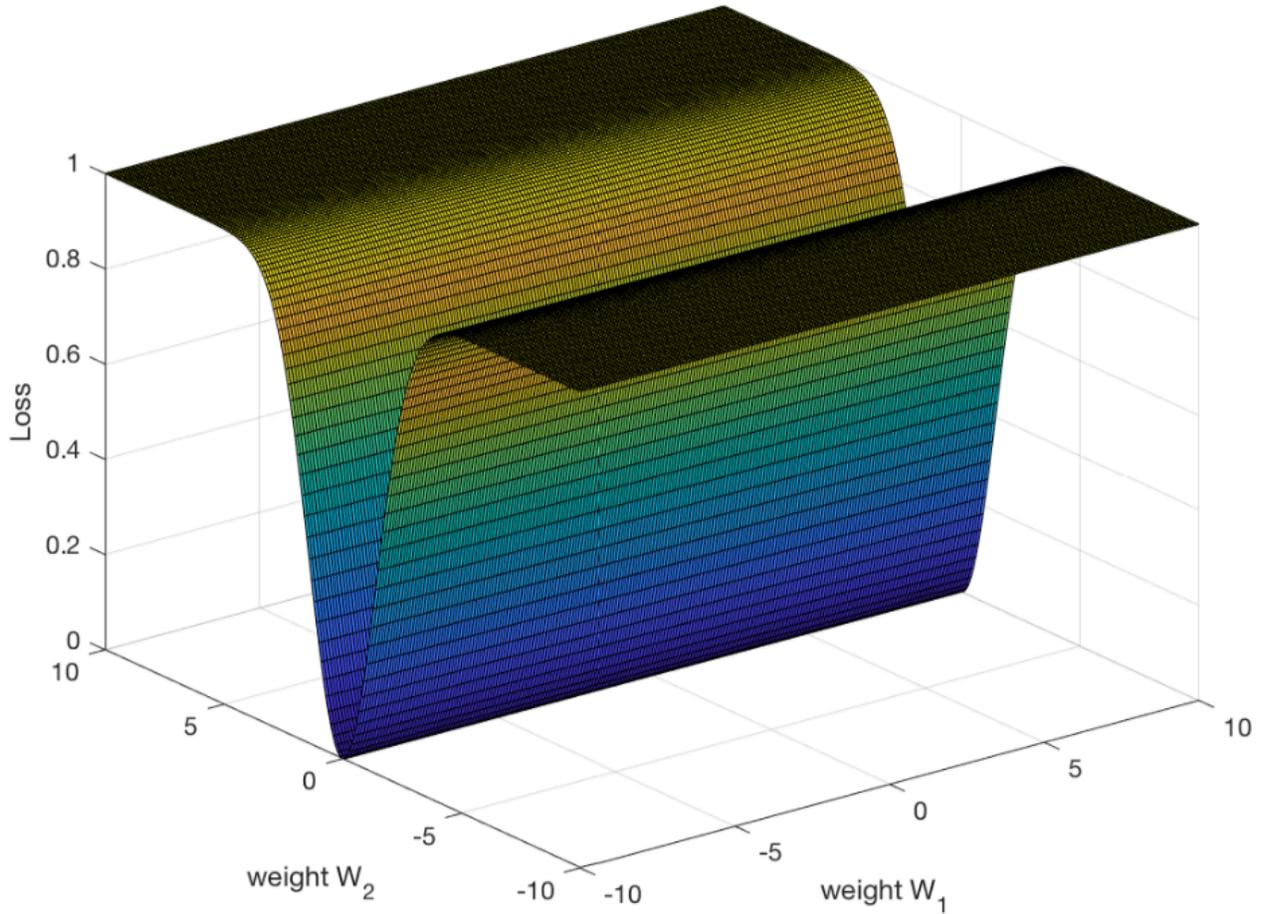}
\caption{\it An illustration of a quadratic loss function which is
  locally quadratic near the minimum in the two parameters $w_1$ and
  $w_2$. The minimum has a degenerate Hessian with a zero
  eigenvalue. In the proposition described in the text, this figure
  represents the ``generic'' situation in a small neighborhood of each
  of the zero minimizers with many zero eigenvalues -- and a few
  positive eigenvalues -- of the Hessian of a nonlinear multilayer
  network. In multilayer networks the loss function is likely to be a
  fractal-like hypersurface with many degenerate global minima, each
  locally similar to a multidimensional version of the degenerate
  minimum shown here. For the crossentropy loss, the degenerate
  valley, instead of being flat, is slightly sloped downwards for
  $||w|| \to \infty$. }
\label{degeneratesquareloss}
\end{figure*}

\subsection{Testing the local structure of global minima}


For a test of our proposition or, more precisely, of its assumptions,
consider the dynamics of a perturbation
$\delta W$ of the solution
\begin{equation}
\dot{\delta W} =-  [F(W^*+\delta W) -F(W^*)] =-(  \mathbf{D} L(W^{*})) \delta W.
\label{linearization}
\end{equation}

\noindent where asymptotic stability in the sense of Liapunov is guaranteed if the sum of the
eigenvalues of $- \mathbf{D} L $ -- where $D L$ is the Hessian of $L$ --
is negative.

Consider now the following experiment. After convergence apply a small
random perturbation with unit norm to the parameter vector, then run
gradient descent until the training error is again zero; this sequence
is repeated $m$ times. Proposition \ref{proposition} makes then the
following predictions:

\begin{itemize}
\item The training error will go back to zero after each sequence of
  GD.
\item Any small perturbation of the optimum $W^*$ will be corrected by
  the GD dynamics to push back the non-degenerate weight directions to
  the original values. Since however the components of the weights in
  the degenerate directions are in the null space of the gradient,
  running GD after each perturbation will not change the weights in those
  directions. Overall, the weights will change in the experiment.

\item Repeated perturbations of the parameters at convergence, each
  followed by gradient descent until convergence, will not increase the
  training error but will change the parameters {\it and} increase
  some norm of the parameters and increase the associated test error. In the
  linear case (see Figure \ref{degeneratesquareloss}) the $L_2$ norm
  of the projections of the weights in the null space undergoes a
  random walk: the increase in the norm of the
  degenerate components should then be proportional to
  $\approx \sqrt{m}$ with a constant of proportionality that depends
  on $\approx \sqrt{\frac{1}{N}}$, where $N$ is the dimensionality of
  the null space.

\end{itemize}

Previous  experiments \cite{Theory_II} showed changes in the
parameters and in the test -- but not in the training -- loss,
consistently with our predictions above, which are further supported
by the numerical experiments of Figures \ref{CIFARclass} and
\ref{CIFARclass_pert}. Notice that the slight increase of the test
loss without perturbation in Figure \ref{CIFARclass} is due to the use
of the crossentropy risk. In the case of crossentropy the almost zero
error valleys of the empirical loss function are slightly sloped
downwards towards infinity, becoming flat only asymptotically (they
also become ``narrower and narrower multidimensional holes'' in terms
of sensitivity to perturbations).

\begin{figure*}[h!]\centering
\includegraphics[width=1.0\textwidth]{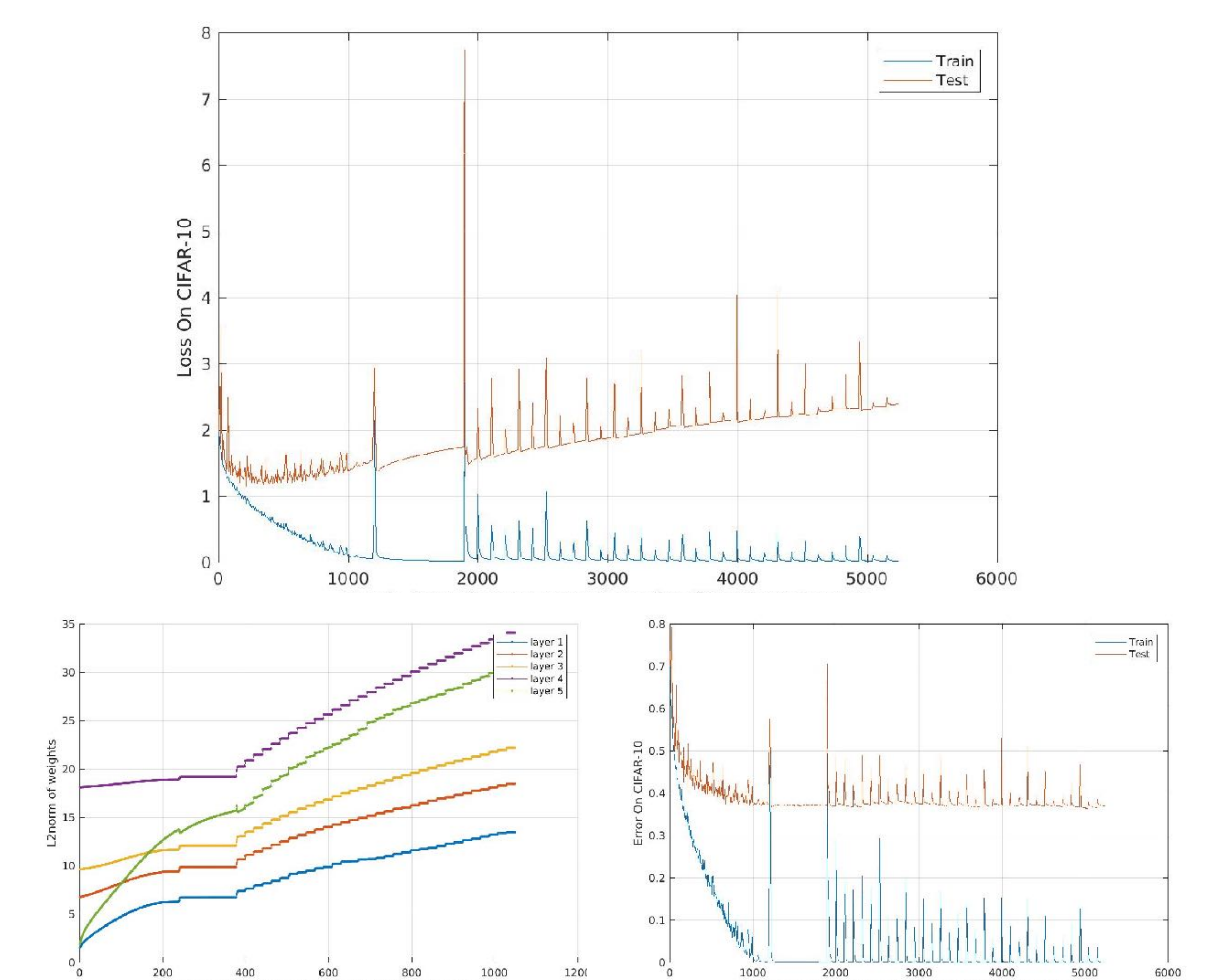}
\caption{\it We train a 5-layer convolutional neural networks on
  CIFAR-10 with Gradient Descent (GD) on crossentropy loss. The top
  plot shows the crossentropy loss on CIFAR during perturbations (see
  text). The bottom left is the corresponding square norm of the
  weights; the bottom right plot shows the classification error (see
  text). The network has 4 convolutional layers (filter size 3×3,
  stride 2) and a fully-connected layer. The number of feature maps
  (i.e., channels) in hidden layers are 16, 32, 64 and 128
  respectively. Neither data augmentation nor regularization is
  performed. Initially, the network was trained with GD as
  normal. After it reaches 0 training classification error (after
  roughly 1800 epochs of GD), a perturbation is applied to the weights
  of every layer of the network. This perturbation is a Gaussian noise
  with standard deviation being $\frac{1}{4}$ of that of the weights
  of the corresponding layer. From this point, random Gaussian noises
  with the such standard deviations are added to every layer after
  every 100 training epochs. Training loss goes back to the original
  level after added perturbations, but test loss grows increasingly
  higher. As expected, the $L_2$-norm of the weights increases after
  each step of perturbation followed by gradient descent which reset
  the taining rror to zero. Compare with Figure \ref{CIFARclass} and
  see text.}
\label{CIFARclass_pert}
\end{figure*}


\begin{figure*}[h!]\centering
\includegraphics[width=1.0\textwidth]{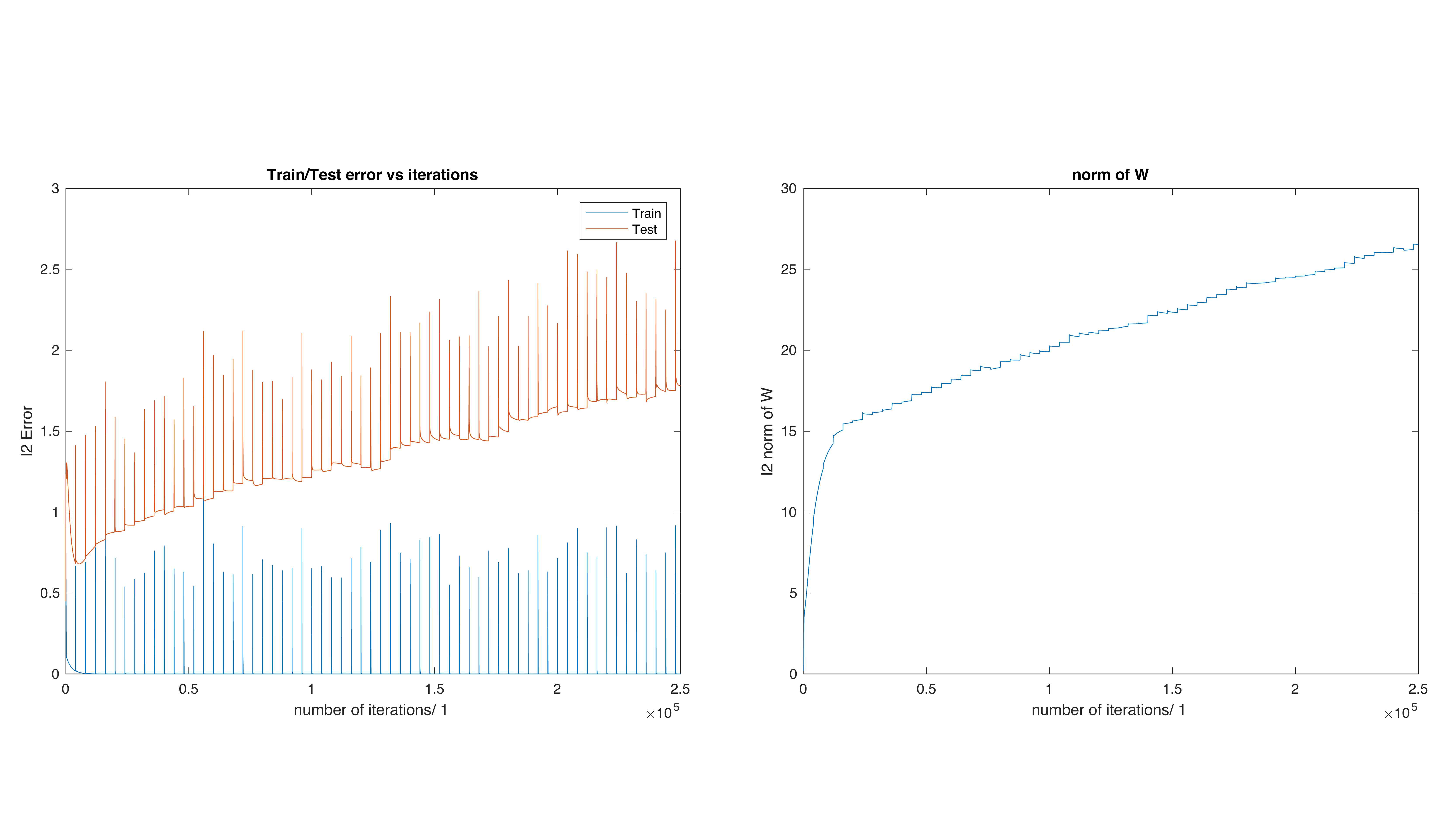}
\caption{\it Training and testing with the square loss for a linear
  network in the feature space (i.e. $y=W\Phi(X)$) with a degenerate
  Hessian of the type of Figure \ref{degeneratesquareloss}.  The
  feature matrix $\phi(X)$ is wrt a polynomial with degree 30.  The
  target function is a sine function $f(x) = sin(2 \pi f x) $ with
  frequency $f=4$ on the interval $[-1,1]$.  The number of training
  points are $9$ while the number of test points are $100$.  The
  training was done with full gradient descent with step size $0.2$
  for $250,000$ iterations.  Weights were perturbed every $4000$
  iterations and then gradient descent was allowed to converge to zero
  training error after each perturbation.  The weights were perturbed
  by addition of Gaussian noise with mean $0$ and standard deviation $0.6$.
  The $L_2$ norm of the weights is shown on the right.  Note that
  training was repeated 30 times. The figure reports the average train
  and test error as well as average norm of the weights over the 30
  repetitions. Figure \ref{Brando1} shows the same situation
  without perturbations.}
\label{Brando}
\end{figure*}

The numerical experiments show, as predicted, that the behavior under
small perturbations around a global minimum of the empirical loss for
a deep networks is similar to that of linear degenerate
regression. Figure \ref{Brando} shows for comparison the latter
case. Notice that in the case of linear degenerate regression (Figure
\ref{Brando1}) the test loss indicates a small overfitting, unlike the
nondegenerate case (Figure \ref{Brando2}. In other words, the minimum
of the test loss occurs at a finite number of iterations. This
corresponds to an equivalent optimum non-zero regularization parameter
$\lambda$ as discussed earlier. Thus a specific ``early stopping'' is
better than no stopping.  The same phenomenon appears for the
nonlinear, multilayer case (see Figure \ref{CIFARclass}) with
crossentropy loss. Note that the corresponding classification loss,
however, does not show overfitting here.

\subsection{Resistance to  overfitting in deep networks}

We discuss here the implications for deep nonlinear networks of the
topological equivalence to linear gradient systems with a degenerate
Hessian as in Figure \ref{degeneratesquareloss}.

For the {\it square loss}, the existence of a non-unstable center
manifold corresponding to the flat valley in Figure
\ref{degeneratesquareloss}, implies that in the worst case the
degenerate weight components under gradient descent will not change
once a global minimum is reached, under conditions in which the
quadratic approximation of Equation \ref{quadraticapproximation} is
locally valid. The weights will be relatively small before convergence
{\it if} the number of iterations up to convergence is small
(consistently with the stability results in
\cite{hardt_train_2015}). In this case 1) the solution may not be too
different from a minimum norm solution 2) overfitting is expected but
can be eliminated by early stopping.

It is interesting that the degeneracy of the Hessian can be eliminated
by an arbitrarily small weight decay which corresponds to transforming
the flat valley into a gentle convex bowl centered in $W^*$.  Consider the dynamics
associated with the regularized loss function
$L_{\gamma}=L+ \gamma ||W||^2$ with $\gamma$ small and positive. Then
the stable points of the corresponding dynamical system will all be
hyperbolic (the eigenvalues of the associated negative Hessian will
all be negative). In this case the Hartman-Grobman
theorem\cite{Wanner2000} holds. It says that the behavior of a
dynamical system in a domain near a hyperbolic equilibrium point is
qualitatively the same as the behavior of its linearization near this
equilibrium point. Here is a version of the theorem adapted to our
case.

 {\it Hartman-Grobman Theorem}  Consider a system evolving in time as $\dot{W} = - F(W)$ with
  $F=\nabla_{W} L(W)$ a smotth map $F: \R^d \to \R^d$. If $F$ has a hyperbolic
  equilibrium state $W^*$ and the Jacobian of $F$ at $W^*$ has no zero
  eigenavlues, then there exist a neighborhood $N$ of $W^*$ and a
  homeomorphism $h: N \to \R^d$, s.t. $h(W^*) =0$ and in $N$ the flow
  of $\dot{W} = - F(W)$ is topologically conjugate by the continous
  map $U=h(w)$ to the flow of  the linearized system $\dot{U}=-HU$
  where $H$ is the Hessian of $L$.

Overfitting is also expected for other loss functions such as the {\it
  logistic (and the cross entropy) loss} that have a global minimum
$L \to 0$ for $||w(t) \to \infty$ reached with a non-zero but
exponentially small slope. In the case of the cross entropy the local
topology is almost the same but the valley of Figure
\ref{degeneratesquareloss} is now gently sloped towards the zero
minimum at infinity.  The situation is thus more similar to the
hyperbolic, regularized situation discussed above. Overfitting of the cross
entropy is shown in Figure \ref{Corrige:GreatPlot}

\begin{figure*}[h!]\centering
\includegraphics[width=1.0\textwidth]{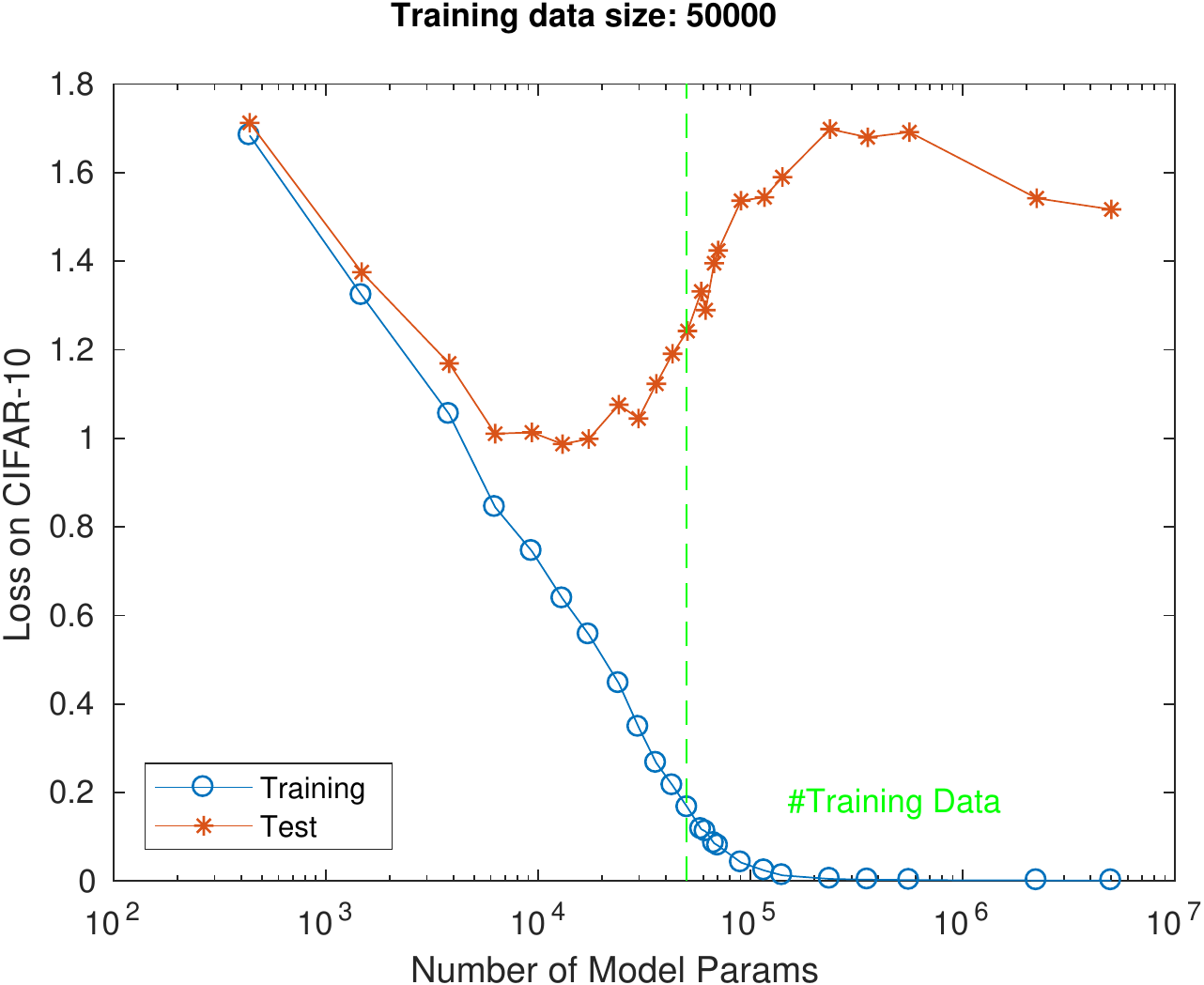}
\caption{\it The previous figures show dependence on $n$ -- number of
  training examples -- for a fixed ReLU architecture with $W$
  parameters.  This figure shows dependence on $W$ of the
  cross entropy loss for a fixed training set of $n$ examples. The
  network is again a 5-layer all convolutional network (i.e., no
  pooling) with ReLUs. All hidden layers have the same number of
  channels. Neither data augmentation nor regularization is
  performed. SGD was used with batch size $= 100$ for $70$ epochs for
  each point. There is clear overfitting in
  the testing loss.}
\label{Corrige:GreatPlot}
\end{figure*}

As in the linear case, however, we can expect a better behavior for
good data sets of the classification error associated with
minimization of the cross entropy \cite{2017arXiv171010345S}. Notice
that, unlike the case of the degenerate square loss, gradient descent
with separable data converges to the max-margin solution with any
starting point $w_0$ (because of the non-zero slope). Thus,
overfitting may not occur at all for the classification error, as
shown in Figure \ref{GreatPlot}, despite overfitting of the associated
cross entropy loss. Notice that the classification error is not the
actual loss which is minimized during training. The loss which is
minimized by SGD is the crossentropy loss; the classification error
itself is invisible to the optimizer.  In the case of Figure
\ref{Corrige:GreatPlot} the crossentropy loss is clearly overfitting
and increasing during the transition from under- to
over-parametrization, while the classification error does not change.
This is the expected behavior in the linear case.

\begin{figure*}[h!]\centering
\includegraphics[width=0.9\textwidth]{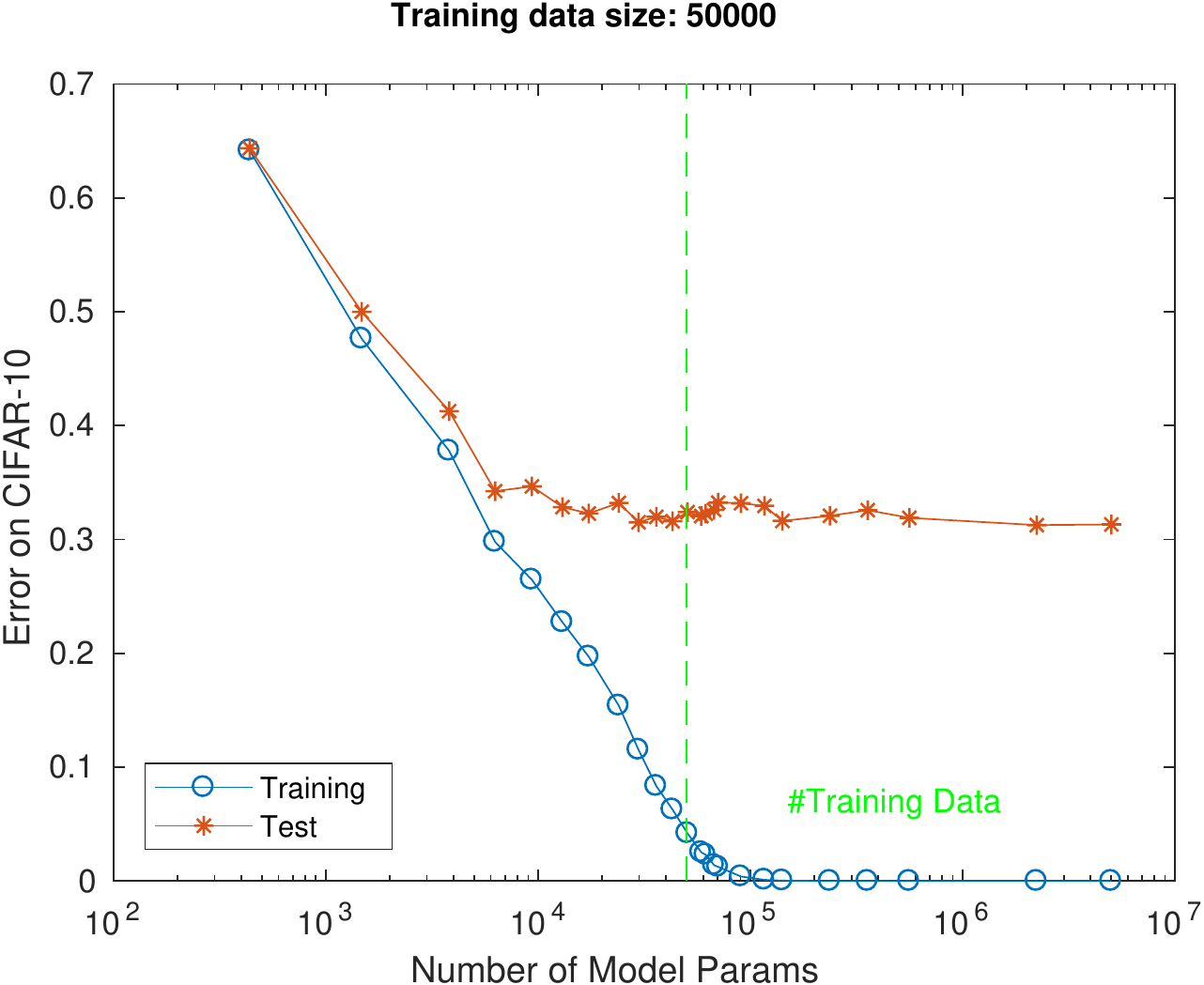}
\caption{\it The figure shows the classification error under the same
  conditions of Figure \ref{Corrige:GreatPlot}. The classical theory explains the generalization behavior
  on the left; the text explains the lack of overfitting for
  $W>N$.}
\label{GreatPlot}
\end{figure*}

In summary, Proposition \ref{proposition} implies that multilayer,
deep networks should behave similarly to linear models for regression
and classification. The theory of dynamical system suggests a
satisfactory approach to explain the central {\it puzzle of non
  overfitting} shown in Figure \ref{GreatPlot} (and Figure
\ref{Corrige:GreatPlot}). Notice that Figures in the SI show the same
behavior for a deep polynomial convolutional network.






\section{Discussion}

Our main contribution is showing that near a global, flat minimum,
gradient descent for deep networks is qualitatively equivalent to a
linear system with a gradient dynamics, corresponding to  a quadratic degenerate
potential.  Proposition \ref{proposition} proves the conjecture in
\cite{RosascoRecht2017} that many properties of linear models should be
inherited by deep networks. In particular, deep networks, similarly to  linear
models, are predicted to  overfit the loss while  not overfitting the classification
error (for nice data sets). Through our reduction to linear from
nonlinear dynamical system via Proposition \ref{proposition},  this follows from
properties of gradient descent for linear network, namely {\it
  implicit regularization} of the loss and {\it margin maximization}
for classification. In practical use of deep networks, explicit
regularization (such as weight decay) together with other regularizing
techniques (such as virtual examples) is usually added and it is often
beneficial but not strictly necessary.

Thus there is nothing magic in deep learning that requires a theory
different from the classical linear one with respect to
generalization, intended as convergence of the empirical to the
expected error, {\it and especially} with respect to the absence of
overfitting in the presence of overparametrization. In our framework,
the fact that deep learning networks optimized with gradient descent
techniques {\it do generalize} follows from the implicit
regularization by GD and from previous results such as
Bartlett's\cite{AntBartlett2002} and
Recht\cite{DBLP:journals/corr/HardtRS15}.  More interestingly,
Proposition \ref{proposition} explains the puzzling property of deep
networks, seen in several situations such as CIFAR, of overfitting the
loss while not overfitting the classification error by showing that
the properties of linear networks emphasized by
\cite{2017arXiv171010345S} apply to deep networks under certain
assumptions on the empirical minimizers and the datasets.

Of course, the problem of establishing quantitative and useful bounds
on the performance of deep network, as well as the question of which
type of norm is effectively optimized, remains an open and challenging
problem (see \cite{DBLP:journals/corr/abs-1711-01530}), as it is
mostly the case even for simpler one-hidden layer networks, such as
SVMs. Our main claim is that the puzzling behavior of Figure
\ref{GreatPlot} can be explained {\it qualitatively} in terms of the
classical theory.

There are of course a number of open problems. Though we explained the
absence of overfitting we did not explain in this paper why deep
networks perform as well as they do. It is quite possible however that
the answer to this question may be already contained in the following
summary of the existing theoretical framework about deep learning,
based on  \cite{Theory_I}, \cite{Theory_II},
\cite{Theory_IIb},  \cite{DBLP:journals/corr/abs-1711-01530} and  this paper:

\begin{itemize}
\item unlike shallow networks deep networks can approximate the class
  of hierarchically local functions without incurring in the curse of
  dimensionality\cite{MhaskarPoggio2016b,Theory_I};
\item overparametrized deep networks yield many global degenerate --
  and thus flat -- or almost degenerate minima\cite{Theory_II} which are
  selected by SGD with high probability\cite{Theory_IIb};
\item overparametrization, which yields overfit of the loss, can avoid
  overfitting the classification error for nice datasets  because of implicit
  regularization by gradient descent methods and the associated 
  margin maximization.

\end{itemize}

According to this framework, the main difference between shallow and deep networks
is in terms of approximation
power. Unlike shallow networks, deep local (for instance
convolutional) networks can avoid the curse of
dimensionality\cite{Theory_I} in approximating the class of
hierarchically local compositional functions. SGD in overparametrized
networks selects with very high probability degenerate minima which
are often global.  As shown in this note, overparametrization does not
necessarily lead to overfitting of the classification error.

We conclude with the following, possibly quite relevant
observation. The robustness of the expected error with respect to
overparametrization suggests that, at least in some cases, the
architecture of locally hierarchical networks may not need to be
matched precisely to the underlying function graph: a relatively large
amount of overparametrization (see Figure \ref{GreatPlotPol}) may not
significantly affect the predictive performance of the network, making
it easier to design effective deep architectures.

{\bf Acknowledgment}

{\it We are grateful for comments on the paper by Sasha Rakhlin and
  for useful discussions on dynamical systems with Mike Shub.  This
  work was supported by the Center for Brains, Minds and Machines
  (CBMM), funded by NSF STC award CCF – 1231216. CBMM 
  acknowledges the support of NVIDIA Corporation with the donation of
  the DGX-1 used in part for this research. HNM is supported in part by ARO
  Grant W911NF-15-1-0385}


\bibliographystyle{ieeetr}
\small
\bibliography{Boolean}
\normalsize

\newpage

\section{Supplementary Information}

\subsection{Implicit/Iterative Regularization by GD and SGD for linear
  dynamical 
systems}
\label{Lorenzo}

We recall some recent results regarding the regularization properties of SGD \cite{LinRos17}.  Here regularization does not arise from explicit penalizations or constraints. Rather it is implicit in the sense that it is induced  by dynamic of the SGD iteration and controlled by the choice of either the step-size, the number of iterations, or both.  More precisely, we recall recent results showing convergence and convergence rates to the minimum expected risk, but also  convergence  to the minimal norm minimizer of the expected risk. 

All the the results are for the least square loss  and assume linearly parameterized functions. With these choices, the learning problem is to solve
$$\min_{w\in {\mathbb R}^D} {\mathcal E}(w)\quad\quad\quad
{\mathcal E}(w)=
\int   (y-{w}^\top{x})^2d \rho$$
given only $(x_i,y_i)_{i=1}^n$ -- $\rho$ { fixed, but unknown}.
Among all the minimizers ${\cal O}$ of the expected risk the minimal norm  solution is 
$w^\dagger$ solving
$$
\min_{w\in {\cal O}} \| {w} \|,
$$
We add two remarks.
\begin{itemize}
\item 
The results we present next  hold considering functions of the form 
$$
f(x)=\langle{w}, {\Phi(x)}\rangle
$$
for some nonlinear feature map $\Phi:{\mathbb R}^D\to {\mathbb R}^p$. 
\item Further, all the results are dimension independent, hence allow to consider $D$ (or $p$) to be infinite.
\end{itemize}

The SGD iteration is given by 
$$
\widehat w _{t+1} =
\widehat w _{t}
-
\eta_t
x_{i_t}(\widehat w _{t}^\top  x_{i_t}-y_{i_t})  , \quad t=0, \dots T
$$
where 
\begin{itemize}
\item  $(\eta_t)_t$ specifies  the step-size, 
\item  $T$ stopping time,  ($T>n$ multiple {\em ``passes/epochs''})
\item   $(i_t)_t$ specifies how the iteration visits the training points.  Classically $(i_t)_t$ is chosen to be stochastic and induced by a uniform distribution over the training points.
\end{itemize}
The following result considers the regularization and learning properties of SGD in terms of the corresponding excess risk.\\
{\bf Theorem}\\
{\em 
Assume $\|x\|\le 1$ and $|y|\le 1$ for all $(\eta_t)_t=\eta$ and $T$, with high probability
$$
{\cal E}(\widehat w _T)- {\cal E}(w^\dagger)
\lesssim \frac{1}{\eta T}+
\frac{1}{\sqrt{n}}\left( \frac{\eta T}{\sqrt{n}}\right)^2 + \eta \left(1 \vee \frac{\eta T}{\sqrt{n}}\right)
$$
In particular, if  $\eta =\frac{1}{\sqrt{n}}$ and $T=n$, then 
$$
\lim_{n\to \infty}{\cal E}(\widehat w _T)= {\cal E}(w^\dagger)
$$
almost surely.
}\\
The next result  establishes convergence to the minimal norm solution of the expected risk.\\
{\bf Theorem}\\
{\em 
Assume $\|x\|\le 1$ and $|y|\le 1$, if $\eta=\frac{1}{\sqrt{n}}$ and $T=n$, then 
$$
\lim_{n\to \infty} \|w _T-w^\dagger\|=0.
$$
almost surely.
}

\subsection{Hessian of  the  loss for multilayer networks}
\label{MultiNets}

\subsubsection{No-hidden-layer linear systems}
\label{MinNorm}

For a linear network without hidden layers with the
training set $(X,Y)$, a
set of weights $A$ can be
found by minimizing

\begin{equation}
\min_A \|Y-AX\|^2_F,
\end{equation}
where  $Y\in \mathcal{R}^{d',n}$ , $A \in \mathcal{R}^{d',d}$,
$X\in \mathcal{R}^{d,n}$ and $n$ is number of data points, 
\noindent yielding an equation for $A$ that is degenerate when $d >
n$. The general solution is
in the form of
\begin{equation*}
A=YX^{\dagger} + M,
\end{equation*}
\noindent where  $X^{\dagger}$ is the pseudoinverse of the matrix $X$, and $M$ is any matrix in the left null space of $X$. The minimum norm solution is
\begin{equation}
A=YX^{\dagger}. 
\label{STsol}
\end{equation}

Furthermore, as shown in the section before, the solution
Equation \ref{STsol} is found by GD and SGD and corresponds to
iterative regularization where the role of the inverse of the regularization
parameter $\lambda$ is played by the product of step size and number
of steps\cite{rosasco2015learning}.

Let us look in more detail at the gradient descent solution (from
\cite{Musings2017}). Consider the following setup:
$X=(x_1,\ldots,x_n)^\top \in\mathbb{R}^{n, d}$ are the data points,
with $d>n$. We further assume that the data matrix is of full row
rank: $\text{rank}(X)=n$.  Let $y\in\mathbb{R}^n$ be the labels, and
consider the following linear system:
\begin{equation}
  Xw=y\label{eq:min-norm-ls}
\end{equation}
where $w\in\mathbb{R}^d$ is the weights to find. This linear system
has infinite many solutions because $X$ is of full row rank and we
have more parameters than the number of equations. Now suppose we
solve the linear system via a least square formulation
\begin{equation}
  L(w) = \frac{1}{2n}\|Xw-y\|^2\label{eq:min-norm-ls-objv}
\end{equation}
by using gradient descent (GD) or stochastic gradient descent
(SGD). In this section, we will show that both GD and SGD converges to
the minimum norm solution, for the $\ell_2$ norm. A similar analysis
(with fewer details) was presented in \cite{zhang2016}.

\begin{lemma}
  The following formula defines a solution to \eqref{eq:min-norm-ls}
  \begin{equation}
    w_{\dagger} \triangleq X^\top (XX^\top)^{-1}y
  \end{equation}
  and it is the minimum norm solution.
  \label{lem:min-norm-lemma1}
\end{lemma}
{\it Proof}
  Note since $X$ is of full row rank, so $XX^\top$ is invertible. By definition,
  \[
  Xw_\dagger = XX^\top (XX^\top)^{-1}y = y
  \]
  Therefore, $w_\dagger$ is a solution. Now assume $\hat{w}$ is
  another solution to \eqref {eq:min-norm-ls}, we show that
  $\|\hat{w}\|\geq \|w_\dagger\|$. Consider the inner product
  \[
  \begin{aligned}
    \langle w_\dagger, \hat{w}-w_\dagger \rangle
    &= \langle X^\top (XX^\top)^{-1}y, \hat{w}-w_\dagger \rangle \\
    &= \langle (XX^\top)^{-1}y, X\hat{w}-X_{\dagger} \rangle \\
    &= \langle (XX^\top)^{-1}y, y-y \rangle \\
    &= 0
  \end{aligned}
  \]
  Therefore, $w_\dagger$ is orthogonal to $\hat{w}-w_\dagger$. As a
  result, by Pythagorean theorem,
  \[
  \|\hat{w}\|^2 = \|(\hat{w}-w_\dagger) + w_\dagger\|^2 = \|\hat{w}-w_\dagger\|^2 + \|w_\dagger\|^2
  \geq \|w_\dagger\|^2
  \]

\begin{lemma}
  When initializing at zero, the solutions found by both GD and SGD for problem \eqref
  {eq:min-norm-ls-objv} live in the span of rows of $X$. In other words, the solutions are of the
  following parametric form
  \begin{equation}
    w = X^\top\alpha
  \end{equation}
  for some $\alpha\in\mathbb{R}^n$.
  \label{lem:min-norm-lemma2}
\end{lemma}
{\it Proof}

  The gradient for \eqref{eq:min-norm-ls-objv} is
  \[
 \nabla_w L(w) = \frac{1}{n}X^\top (Xw-y) = X^\top e
  \]
  where we define $e=(1/n)(Xw-y)$ to be the error vector. GD use the following update rule:
  \[
  w_{t+1} = w_t - \eta_t \nabla_w L(w_t) = w_t - \eta_t X^\top e_t
  \]
  Expanding recursively, and assume $w_0=0$. we get
  \[
  w_t = \sum_{\tau=0}^{t-1} -\eta_\tau X^\top e_{\tau}
  = X^{\top}\left( -\sum_{\tau=0}^{t-1} \eta_\tau e_{\tau} \right)
  \]
  The same conclusion holds for SGD, where the update rule could be explicitly written as
  \[
  w_{t+1} = w_t - \eta_t (x_{i_t}^\top w - y_{i_t})x_{i_t}
  \]
  where $(x_{i_t},y_{i_t})$ is the pair of sample chosen at the $t$-th iteration. The same
  conclusion follows with $w_0=0$.

Q.E.D.

\begin{theorem}
  Let $w_t$ be the solution of GD after $t$-th
  iteration, and $\mathsf{w}_t$ be the (random) solution of SGD after $t$-th iteration. Then $\forall
  \varepsilon>0$, $\exists T>0$ such that
  \[
  L(w_t) \leq \varepsilon,\quad \mathbb{E}L(\mathsf{w}_t) \leq \varepsilon
  \]
  where the expectation is with respect to the randomness in SGD.
  \label{thm:min-norm-gd-converge}
\end{theorem}

\begin{corollary}
When initialized with zero, both GD and SGD converges to the minimum-norm solution.
\end{corollary}
{\it Proof}
  Combining Lemma~\ref{lem:min-norm-lemma2} and Theorem~\ref{thm:min-norm-gd-converge}, GD is
  converging $w_t \rightarrow w_\star=X^\top\alpha_\star$ as $t\rightarrow \infty$ for some optimal
  $\alpha_\star$. Since $w_\star$ is a solution to \eqref{eq:min-norm-ls}, we get
  \[
  y = Xw_\star = XX^\top \alpha_\star
  \]
  Since $XX^\top$ is invertible, we can solve for $\alpha_\star$ and get
  \[
  w_\star = X^\top \alpha_\star = X^\top (XX^\top) y = w_\dagger
  \]
  Similar argument can be made for SGD with expectation with respect to the randomness of the
  algorithm.

  In the degenerate linear approximation of a nonlinear system near a
  global minimum the above situation holds but the null space is not
  guaranteed to remain unaffected by early steps in GD. Norms will
  remain bounded if the gradient is bounded (in fact the title of
  \cite{DBLP:journals/corr/HardtRS15} {\it Train faster, generalize
    better}...'' may apply here) and thus generalization is
  guaranteed.  The solution is {\it not} strictly guaranteed to be of
  minimum norm in the case of quadratic loss. It is, however,
  {\it guaranted to be the maximum margin solution for separable data
    in the case of logistic and croos-entropy loss}
  \cite{2017arXiv171010345S}.

\subsubsection{One-hidden layer linear networks}
\label{DeepNetsLinSusubsection}g

Consider  linear
network of the same type analyzed by \cite{Baldi:89,
  DBLP:journals/corr/SaxeMG13} in the full rank case and shown in
Figure \ref{1-hidden}. This is a network with an input layer, an
output layer, a hidden layer of {\it linear} units and two sets of
weights. Our results  for the degenerate case are new.

\begin{figure*}[h!]\centering
  \includegraphics[width=0.4\textwidth]{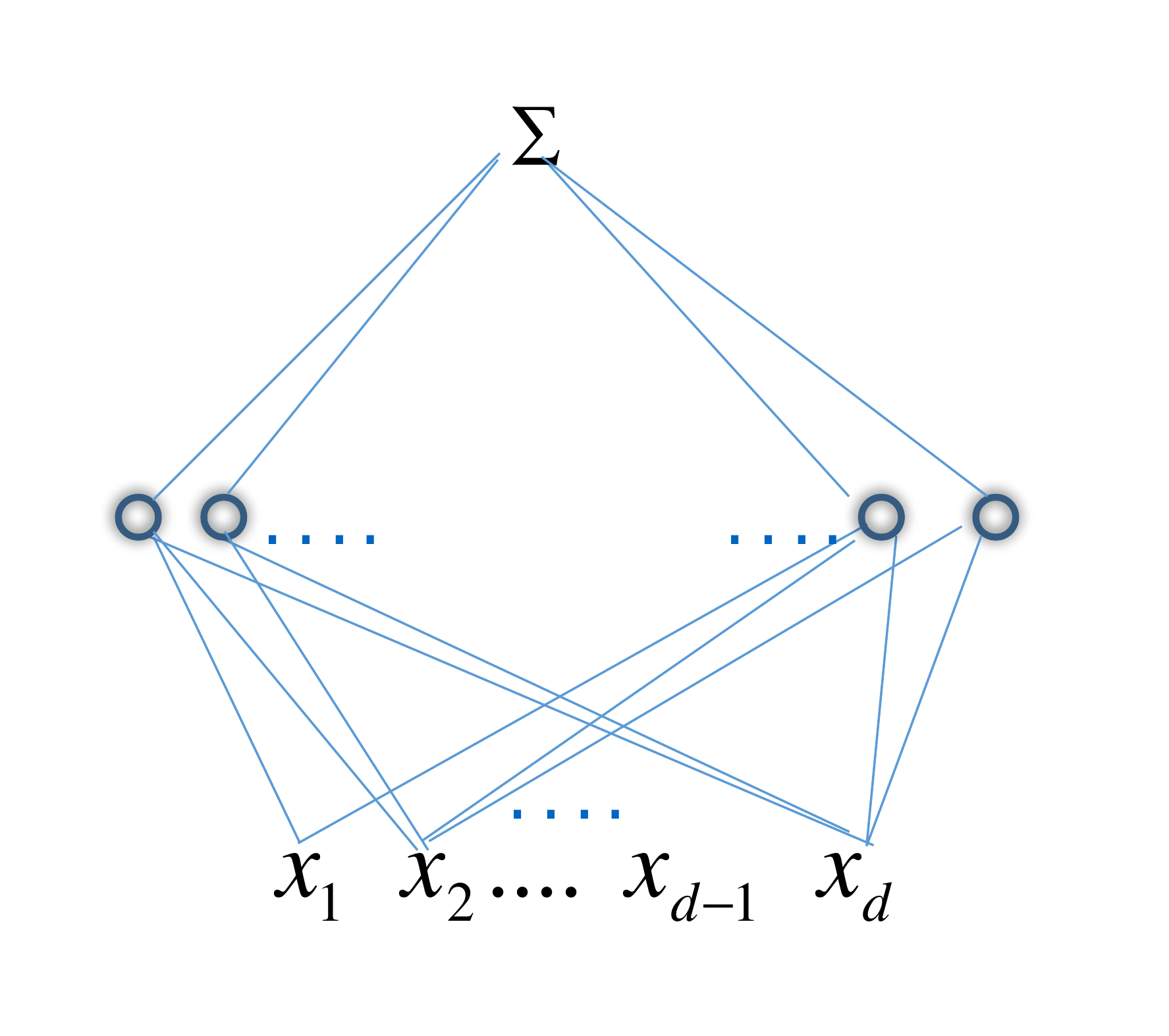}
\caption{\it A one-hidden layer network with weights represented by
  the matrix $W_1$ from the $d$ inputs to the $N$ hidden units and the
  matrix $W_2$ for the weights from the hidden units to the output.}
\label{1-hidden}
\end{figure*}

Suppose that gradient descent is used to minimize a quadratic loss on
a set of $n$ training examples. The activation function of the hidden
units is linear, that is $h(z)=az$. Notice that though the mapping
between the input $x$ and the output $f(x)$ is linear, the dependence
of $f$ on the set of weights $W_1, W_2$ is nonlinear. Such deep linear
networks were studied in a series of papers\cite{Baldi:89,
  DBLP:journals/corr/SaxeMG13, kawaguchi2016deep} in situations in
which there are more data than parameters, $n \ge d$. In this section
we show how the network will not overfit even in the
overparametrized case, which corresponds to $n < d$.

\begin{figure*}[h!]\centering
\includegraphics[width=1.0\textwidth]{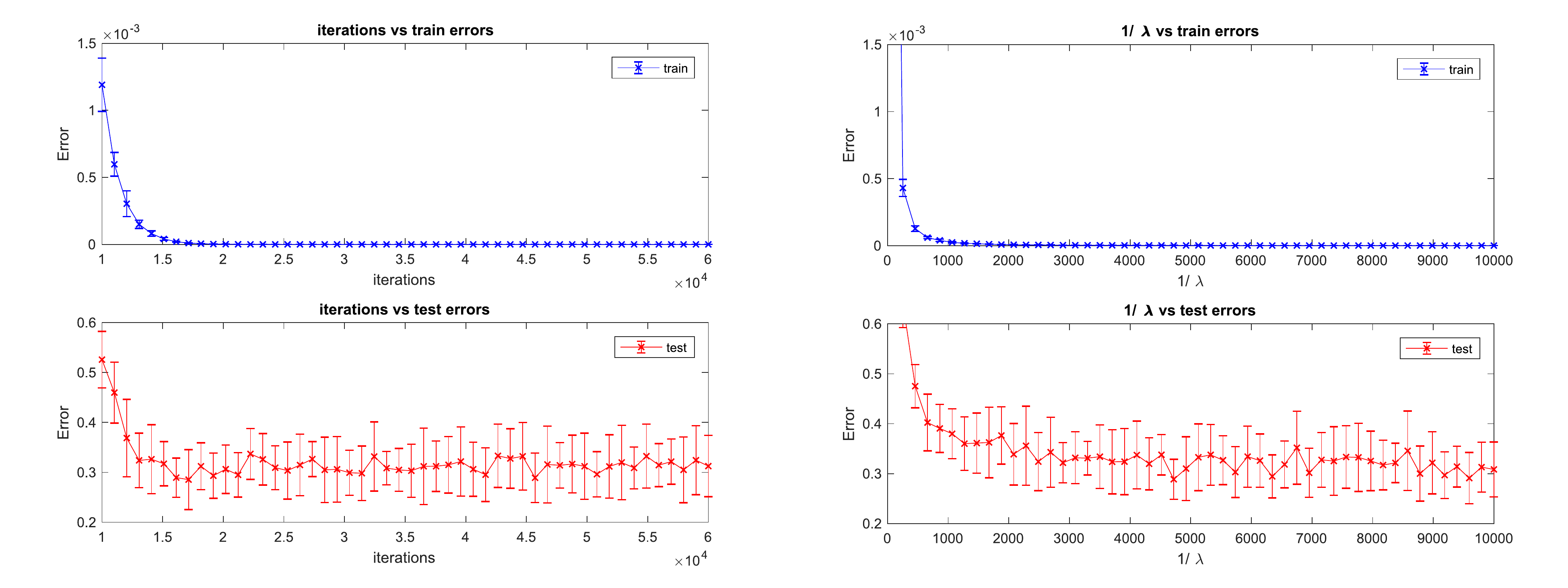}
\caption{\it In this figure we show how a linear network
  $W_2 W_1 X = y $ converges to the test error in two different
  training protocols: in the first the number of iterations
  grows, in the second  $\frac{1}{\lambda}$ grows large (with
  optimization run to covergence for each $\lambda$.  Optimization is
  performed in both case with  SGD. The training on the right
  were done on by minimizing 
  $\frac{1}{n} \sum^{n}_{i=1} V(f(x),y) + \lambda \| W_2 W_1 \| $
  with various values of the regularization parameters
  $\lambda$. Similarly, the experiments on the right minimized the
  same cost function except that $\lambda = 0$.  The number of data
  points were $n=30$, the input dimension $d=31$ (including the
  bias).  The target function was generated by a two layer linear network with 1
  hidden unit and thus a total of 32 weights. The approximating two linear network 
  had 2 hidden units and thus 64 parameters.  The test set had 32 data
  points. Note that there was no noise in the training or test set.}
\label{1-hiddenLin}
\end{figure*}

\paragraph{Model Description and notation}
Consider a network, $d$ inputs, $N$
hidden {\it linear} units and $d'$ outputs. 
We denote
the loss with $L(w)=\frac{1}{2}||W_2W_1 X-Y||^2$, where
$X \in \R^{d, n}, Y \in \R^{d',n}$, $W_2 \in \R^{d',N}$ and
$W_1 \in \R^{N,d}$. Let $E = W_2 W_1 X -Y \in \R^{d',n}$. Let $w=\vect(W_1^\top, W_2^\top)\in \R^{Nd+d'N}$.

\paragraph{Gradient dynamics and Hessian}

We can write the dynamical system corresponding to its gradient as

\begin{equation}
\dot{W_1} = - \nabla_{W_1} L(w)= - W_2^\top  E X^\top =
W_2^\top YX^\top -  W_2^\top W_2W_1 X X^\top
\label{1-hidden-dynamics-W_1}
\end{equation}

\noindent and similarly

\begin{equation}
\dot{W_2}= - YX^\top W_1^\top +  W_2W_1XX^\top W_1^\top
\label{1-hidden-dynamics-W_2}
\end{equation}

The linearization of this dynamical system (Equations
\ref{1-hidden-dynamics-W_1} and \ref{1-hidden-dynamics-W_2})
corresponds to the Hessian of the  loss, which is 
$$
 \nabla^2 L(w) =
\begin{bmatrix}
W^{\top}_2W_2 \otimes XX^\top & C \\
C^{\top} & I_{d'} \otimes XX^\top W_1 XX^\top W_1 \\
\end{bmatrix}
\in \R^{(Nd+d'N), (Nd+d'N)}
$$
where
$$
C = [W_{2}^\top \otimes XX^\top W_1^\top ]+  [I_N \otimes X (E^\top)_{\bullet 1}, \dots, I_N \otimes X (E^\top)_{\bullet d'}].
$$
Here, $(E^\top)_{\bullet i}$ denote the $i$-th column of $E^\top$. 

\paragraph{Degeneracy of the dynamics}

Suppose that $d' <N$ (overparametrization). Then, $W^{\top}_2W_2$ has zero eigenvalue, which
implies that $W^{\top}_2W_2 \otimes XX^\top$ has zero eigenvalue. In turn, this 
implies that $\nabla^2 L(w)$ has zero eigenvalue. Indeed, we can
construct a degenerate direction as follows: consider a vector
$w=([w_1\otimes w_2],0)^\top$ such that $w_1$ is in the null space of
$W_2$, with which we have that $w^\top\nabla^2 L(w)w=0$.

However, along this degenerate direction, the output model does not
change at all, because this direction corresponds to adding an element
of null space of $W_2$ to $W_1$ (i.e., the product $W_2W_1$ does not
change in this direction) Also, (S)GD does not add element
corresponding to this direction from Equation
\ref{1-hidden-dynamics-W_1}, since $\dot{W_1}$ is in the column space
of $W_2^\top$ which is orthogonal to the null space of $W_2$.

\paragraph{Trivial  and nontrivial degeneracy} 

It is interesting that even in the overdetermined case $n>>d$, the
solution is not unique in terms of $W_1$ and $W_2$.  This corresponds
to the ``trivial'' degeneracy, which is a terminology defined as
follows: Let $f_w: \mathcal{X} \rightarrow \mathcal{Y}$ be a model
parameterized by the parameter $w$. For example, in the case of the
linear one-hidden layer model, $f_w(x)=W_2W_1 x$ where
$w=\vect(W_1^\top, W_2^\top)$ and $x \in \mathcal X$. Then, we call
degeneracy ``trivial'' if and only if there exists $\epsilon >0$ such
that $f_w=f_{w+\epsilon \Delta w}$ where $\Delta w$ is in the
degenerate space. Here, we write $f_1=f_2$ if and only if their graphs
are equal: i.e.,
$\{(x,y) \in \mathcal X \times \mathcal Y: x \in \mathcal X,
f_1(x)=y\}=\{(x,y) \in \mathcal X \times \mathcal Y: x \in \mathcal X,
f_2(x)=y\}$. We call degeneracy ``nontrivial'' if and only if it is
not trivial.

As we will see, trivial degeneracy occurs in multilayer networks with
unit activation which is linear (section
\ref{DeepNetsLinSusubsection}) or ReLU (section
\ref{sec:app_nonlinear_nets}).  The situation in terms of trivial and
nontrivial degeneracy is quite different in the overdetermined case
and underdetermined case:

\begin{itemize}
\item The solution in the overdetermined case yields a network that is
  unique despite the degeneracy. In other words, every possible
  degeneracy is \textit{trivial degeneracy}. Furthermore, more data
  cannot change this situation.\item In the underdetermined case, we
  may have both types of degeneracy. More data would avoid
  \textit{nontrivial degeneracy} which yields different networks.
\end{itemize}


\paragraph{GD and SGD converge to the Minimum Norm Solution }
Assume that $N,d \ge n \ge d'$ (overparametrization).  Let
$A=W_2W_1$. For any matrix $M$, let $\Col(M)$ and $\Null(M)$ be the
column space and null space of $M$.

Due to the convexity of the loss $L$ as a function of the entries of
$A$, the global minimum $A^*$ in terms of the entries of $A$ is
obtained as follows: \begin{align*} &\frac{\partial L}{\partial A^*} =
  (A^*X-Y) X^\top = 0 \\ & \Leftrightarrow A^* \in \{A= YX
  ^{\top}(XX^{\top})^{\dagger}+ B_X : B_{X}X=0 \}
\end{align*}

The minimum norm solution $A^*_{\min}$ is the global minimum $A^*$
with its rows not in $\Null(X^\top)$, which is
$$
A^*_{\min} =  YX ^{\top}(XX^{\top})^{\dagger}.
$$ 
Due to the over-parameterization with $w$, for any entries of $A$,
there exist entries of $W_{1}$ and $W_2$ such that $A=W_2W_1$. Thus,
the global minimum solutions in terms of $A$ and $w$ are the same.

\begin{lemma} \label{lemma:1-hidden_minimum_norm}
For gradient descent  and stochastic gradient descent with any mini-batch size, 
\begin{itemize}
\item 
any number of the iterations adds no element in $\Null(X^\top)$ to the rows of $W_1$, and hence
\item if the rows of $W_1$ has no element in $\Null(X^\top)$ at
  anytime (including the initialization), the sequence converges to a
  minimum norm solution if it converges to a solution.
\end{itemize}
\end{lemma}
\begin{proof}
From $\frac{\partial L}{\partial \vect(W_1)} =[X \otimes W_2^{\top}]\vect(E) = \vect(W_2^\top E X^\top)$, we obtain that

$$
\frac{\partial L}{\partial W_1} = W_2^\top E X^\top.
$$

For SGD with any mini-batch size, let $\bar X_t$ be the input matrix
corresponding to a mini-batch used at $t$-th iteration of SGD. Let
$\bar L_t$ and $\bar E_t$ be the corresponding loss and the error
matrix. Then, by the same token,

$$
\frac{\partial \bar L_t}{\partial W_1} = W_2^\top \bar E_t \bar X^\top_t.
$$

From these gradient formulas, the first statement follows by noticing
that for any $t$,
$\Col(\bar X_t ) \subseteq\Col(X)\perp \Null(X^{\top})$. The second
statement follows the fact that if the rows of $W_1$ has no element in
$\Null(X^\top)$ at anytime $t$, then for anytime after that time $t$,
$\Col(W_1^\top W_2^\top) \subseteq\Col(W_1^\top )
\subseteq\Col(X)\perp \Null(X^{\top})$.

\end{proof}

The above lemma shows that if (stochastic) gradient descent sequence
converges to a solution, then we can easily ensure the solution to be
a minimum norm solution. However, we have not shown that (stochastic)
gradient descent  converges to a global minimum solution. This result
is provided by the following lemma.
 
\begin{lemma} \label{lemma:1-hidden_stationary_equal_global}
If $W_2 \neq 0$, every stationary point w.r.t. $W_1$ is a global minimum.
\end{lemma}
\begin{proof}

For any global minimum solution $A^*$, the transpose of the model output is
$$
(A^{*}X)^{\top}=X^\top (XX^{\top})^{\dagger} XY^\top
$$
\noindent which is the projection of $Y^\top$ onto $\Col(X^{\top})$. Let $D=[W_{2}\otimes X^\top]$. Then, with the transpose of the model output,
the loss can be rewritten as $$
L(w) =\frac{1}{2} \|X^{\top} W_1^\top W_2^\top-Y^\top\|^2= \frac{1}{2} \|D\vect(W_{1}^\top)-Y^\top \|^2.
$$ 
The condition for a stationary point yields
\begin{align*}
  &0=\frac{\partial L}{\partial \vect(W_{1}^\top)} =D^{T}(D \vect(W_1^\top)-Y^\top) \\ & \Rightarrow D\vect(W_{1}^\top)=D(D^TD)^{\dagger} D^T Y^\top  = \text{Projection of $Y^\top$ onto $\Col(D)$}. 
\end{align*}
If $W_2 \neq 0$, we obtain $\Col(D)=\Col(X^{\top})$. Hence
any stationary point w.r.t. $W_1$ is a global minimum. 
\end{proof}

It follows from Lemmas \ref{lemma:1-hidden_minimum_norm} and
\ref{lemma:1-hidden_stationary_equal_global} that, if we initialize
the rows of $W_1$ with no element in $\Null(X^\top)$, and if
$W_2 \neq 0$, GD and SGD find a minimum norm solution.

The summary is provided by 

\begin{theorem} \label{theorem}
If 
\begin{itemize}
\item the rows of $W_1$ are initialized with no element in $\Null(X^\top)$,
\item and if $W_2 \neq 0$, 
\end{itemize}
then gradient descent (and stochastic gradient descent) converges to a minimum norm
solution.
\end{theorem}

As a final remark, the analysis above holds for a broad range of loss
function and not just the square loss.  Asymptotically
$\dot{W} = - \nabla_{W} L(W^\top X) $, in which we makes explicit the
dependence of the loss $L$ on the linear function $W^\top X$. Since
$\nabla_{W} L(W^\top X) = X^\top \nabla_{Z} L(Z)$ the update to $W$ is
in the span of the data $X$, that is, it does not change the
projection of $W$ in the null space of $X$. Thus if the norm of the
components of $W$ in the null space of $X$ was small at the beginning
of the iterations, it will remain small at the end. This means that
among all the solutions $W$ with zero error, gradient descent selects
the minimum norm one.

\subsubsection{One-hidden layer polynomial networks}
\label{PolActiv}

Consider a polynomial activation for the hidden units. The case of
interest here is $ n > d$. Consider the loss
$L(w)=\frac{1}{2}||P_m(X)-Y||^2_F$ where
\begin{equation}
P_m(X) =W_2(W_{1}X)^{\tiny m}.
\label{PolyDeep}
\end{equation}

\noindent where the power $m$ is elementwise.

We obtain, denoting $E = P_m(X)-Y$, with $E \in \R^{d',n}$, $W_2 \in
\R^{d',N}$, $W_1 \in \R^{N,d}$, $E' \in \R^{d', d}$

\begin{equation}
\nabla_{W_1} L(w)=m (W_{1}X)^{m-1}\circ (W_{2}^TE)]
X^\top= E' X^\top
\label{PolyDeep2}
\end{equation}
\noindent where the symbol $\circ$ denotes Hadamard (or entry-wise) product. In
a similar way, we have 

\begin{equation}
\nabla_{W_2} L(w)=E (((W_{1}X)^{m})^\top ).
\label{PolyDeep3}
\end{equation}

\subsubsection{Nonlinear deep networks} 
\label{sec:app_nonlinear_nets}
We now
discuss an extension of the above arguments to the nonlinear
activation case with activation function $\sigma(z)$, such as the ReLU activation.

\paragraph{Gradient dynamics  of polynomial multilayer networks}   We remark that
if  $\sigma(z)=az+bz^2$, the dynamical system induced by GD and the square loss
is 

\begin{equation}
\dot{W_1}=-   (a W_2^\top  E  + 2 b [(W_{1}X) \circ (W_{2}^TE)])
X^\top
\label{PolyDeep21}
\end{equation}
\noindent and

\begin{equation}
\dot{W_2}=- [aE X^\top W_1^\top + bE (((W_{1}X)^{2})^\top )].
\label{PolyDeep31}
\end{equation}

\paragraph{Hessian and degeneracy of  nonlinear one-hidden layer networks}   

For a general activation function $\sigma,$  some block of the Hessian of
the loss can be written as

$$
\nabla^2 L(w) = \begin{bmatrix}- & - \\
- & I_{d'} \otimes \sigma(W_1X) \sigma(W_1X)^\top\ \\
\end{bmatrix}.
$$

We can derive other blocks of the Hessian  as
follows:
$\nabla^2 L(w) = \nabla \hat Y(w) \nabla \hat Y(w)^\top + \sum_{i=1}^n
\nabla^2 \hat Y_i(w) E_i$, where $\hat Y(w)=\vect(W_{2}\sigma(W_{1}X))$ is the model output vector.  {\bf If} $E=0$ at the limit point, then
$\nabla^2 L(w) =\nabla \hat Y(w) \nabla \hat Y(w)^\top$. So, let us
consider the case where $E=0$ at the limit point.
\begin{align*}
&\nabla \hat Y(w) \nabla \hat Y(w)^\top 
\\ &= \begin{bmatrix}\nabla_1 \hat Y(w) \nabla_1 \hat Y(w)^\top & \nabla_1 \hat Y(w) \nabla_2 \hat Y(w)^\top \\
\nabla_2 \hat Y(w) \nabla_1 \hat Y(w)^\top & \nabla_2 \hat Y(w) \nabla_2 \hat Y(w)^\top \ \\
\end{bmatrix}.
\end{align*}
where
$$
\nabla_1 \hat Y(w)^\top = [I_{n} \otimes W_2] \dot\sigma(\vect[W_1X]) [X^\top \otimes I_N], 
$$
and
$$
\nabla_2 \hat Y(w)^\top =I_{d'} \otimes  \sigma(W_1X)^\top
$$

We now show that the first block of the Hessian,
$$
\nabla_1 \hat Y(w) \nabla_1 \hat Y(w)^\top =[X \otimes I_N] \dot\sigma(\vect[W_1X]) [I_{n} \otimes W_2^\top W_2] \dot\sigma(\vect[W_1X]) [X^\top \otimes I_N],
$$
has the zero eigenvalues when the model is over-parameterized as  $dN > n \cdot \min(N,d')$. This happens with
over-parameterization. Then, the rank of
$[I_{n} \otimes W_2^\top W_2]$ is at most $n \cdot \min(N,d')$, which
implies that the first block in the Hessian
($[X \otimes I_N] \dot\sigma(\vect[W_1X]) [I_{n} \otimes W_2^\top W_2]
\dot\sigma(\vect[W_1X]) [X^\top \otimes I_N]$) of size $dN$ by $dN$ has zero
eigenvalue (the rank of the product of matrices is at most the minimum
of the ranks of matrices), which implies that
$\nabla \hat Y(w) \nabla \hat Y(w)^\top$ has zero eigenvalue.

\paragraph{Hessian and degeneracy of  nonlinear two-hidden layer networks}

For networks with two-hidden layers, we have that 

$$
\nabla \hat Y(w) \nabla \hat Y(w)^\top = \begin{bmatrix}\nabla_{1} \hat Y(w)\nabla_{1} \hat Y(w)^\top  & \nabla_{1} \hat Y(w)\nabla_{2} \hat Y(w)^\top & \nabla_{1} \hat Y(w)\nabla_{3} \hat Y(w)^\top \\
\nabla_{2} \hat Y(w)\nabla_{1} \hat Y(w)^\top & \nabla_{2} \hat Y(w)\nabla_{2} \hat Y(w)^\top  &   \nabla_{2} \hat Y(w)\nabla_{3} \hat Y(w)^\top \\
\nabla_{3} \hat Y(w)\nabla_{1} \hat Y(w)^\top & \nabla_{3} \hat Y(w)\nabla_{2} \hat Y(w)^\top & \nabla_{3} \hat Y(w)\nabla_{3} \hat Y(w)^\top \\
\end{bmatrix},
$$

where, 
$$
\nabla_{1} \hat Y(w)^\top =[I_{n} \otimes W_{3}] \dot \sigma(h_2)[I_{n} \otimes W_{2}]  \dot \sigma(h_1) [X^\top \otimes I_{N_1}], 
$$
$$
\nabla_{2} \hat Y(w)^\top =[I_{n} \otimes W_{3}] \dot \sigma(h_2)
[\sigma(W_1X)^\top \otimes I_{N_2}],
$$
and 
$$
\nabla_{3} \hat Y(w)^\top =[\sigma(W_2\sigma(W_1X))^\top \otimes I_{N_{3}}].
$$

By the same token, we have zero eigenvalue in the block $\nabla_{1} \hat Y(w)\nabla_{1} \hat Y(w)^\top$ with overparametrization. 

Let
$D^\top=[I_{n} \otimes W_{3}] \dot \sigma(h_2)[I_{n} \otimes W_{2}]
\dot \sigma(h_1) [X^\top \otimes I_{N_1}]$. If the nonlinear
activation is ReLU, then we can write $\hat Y(w)=D^{\top} \vect(W_1)$
and $\nabla_{1} \hat Y(w)\nabla_{1} \hat Y(w)^\top = DD^\top $. Thus,
the zero eigenvalues of the Hessian block
$\nabla_{1} \hat Y(w)\nabla_{1} \hat Y(w)^\top$ are the exactly those
defining the nulls space of $D^\top$. In other words, the zero
eigenvalue directions do not change the loss values. Furthermore,
similarly to the no-null-element proof for deep linear case,
$$
\nabla_{\vect(W_1)} \hat L(w) =D (\hat Y(w)-Y)  
$$ 
which implies that the gradient descent does not add any elements from the null
space of $D^\top$. However, $D$ is now a function of $w$ (and $X$), and
$D$ is changing during training when far  from the critical point.

\paragraph{Derivation of the Hessian}

We describe below the derivation of the Hessian for nonlinear two-hidden
layer case. By the same token, we derived the Hessian for one hidden
layer with nonlinear and linear cases.

To avoid introducing different symbols for different domains of
activation functions, let $\sigma(M)$ be the elementwise application
of nonlinear activation for any size of matrix $M$. Accordingly, given
a size $t$ of vector $v$, let $\dot \sigma(v)$ be the diagonal matrix
of size $t$ by $t$, with the diagonal elements being the derivative of
the nonlinear activation function for each element $v_k$ for
$k=1\dots,t$.

Let $L(w)=\frac{1}{2} \|W_3\sigma(W_2\sigma(W_1X))-Y\|^2_F$. Let $\hat Y_n(w)= W_3\sigma(W_2\sigma(W_1X))$ and $E=\hat Y_n(w)-Y$. Then, $L(w)=\vect(E)^\top \vect(E)$.    

Then, 
\begin{align*}
\vect(\hat Y_n(w) &)= \vect(W_3\sigma(W_2\sigma(W_1X)))
\\ &= [I_{n} \otimes W_{3}] \vect(\sigma(W_2\sigma(W_1X)))
\\ &= [I_{n} \otimes W_{3}] \sigma(\vect(W_2\sigma(W_1X)))
\\ &= [I_{n} \otimes W_{3}]  \sigma([I_{n} \otimes W_{2}]  \vect(\sigma(W_1X)))
\\ &=  [I_{n} \otimes W_{3}]  \sigma([I_{n} \otimes W_{2}]  \sigma(  \vect(W_1X)))
 \\ &=[I_{n} \otimes W_{3}]  \sigma([I_{n} \otimes W_{2}]  \sigma([X^\top \otimes I_{N_1}] \vect(W_1)))  
\end{align*}  
Similarly, 
\begin{align*}
\vect(\hat Y_n(w)) &=[I_{n} \otimes W_{3}] \sigma([\sigma(W_1X)^\top \otimes I_{N_2}]  \vect(W_2)), 
\end{align*}  
and
\begin{align*}
\vect(\hat Y_n(w)) &= [\sigma(W_2\sigma(W_1X))^\top \otimes I_{N_{3}}] \vect(W_3). \end{align*}  

Let $h_2=\vect(W_2\sigma(W_1X))$ and $h_1=  \vect(W_1X))$. Let $\nabla_{k} \vect(\hat Y_n(w))^\top = \nabla_{\vect(W_k)} \vect(\hat Y_n(w))^\top$.

Then,
\begin{align*}
\nabla_{1} \vect(\hat Y_n(w))^\top =[I_{n} \otimes W_{3}] \dot \sigma(h_2)[I_{n} \otimes W_{2}]  \dot \sigma(h_1) [X^\top \otimes I_{N_1}], 
\end{align*}
\begin{align*}
\nabla_{2}\vect(\hat Y_n(w))^\top =[I_{n} \otimes W_{3}] \dot \sigma(h_2) [\sigma(W_1X)^\top \otimes I_{N_2}], 
\end{align*}
and 
\begin{align*}
\nabla_{3} \vect(\hat Y_n(w))^\top =[\sigma(W_2\sigma(W_1X))^\top \otimes I_{N_{3}}].
\end{align*}

\paragraph{Degeneracy  for nonlinear k-layer networks}

The following statement generalizes the above results about degeneracy of nonlinear two-hidden layer networks to those of nonlinear multilayer networks with an arbitrarily depth.
We use the same notation as above unless otherwise specified.
\begin{theorem} \label{thm:zero_eigenvalue}
Let $H$ be a positive integer. Let 
$
h_k =W_{k} \sigma(h_{k-1}) \in \mathbb R^{N_k,n}
$ 
for $k \in \{2,\dots,H+1\}$ and $h_1=W_{1}X$, where  $N_{H+1}=d'$.  Consider a set of $H$-hidden layer models of the form,
$
\hat Y_n(w) = h_{H+1}, 
$  
parameterized by $w=\vect(W_1,\dots,W_{H+1})\in \mathbb{R}^{dN_1+N_1N_2+N_2N_3+\cdots+N_HN_{H+1}}$. Let $L(w)=\frac{1}{2} \|\hat Y_n(w) -Y\|^2_F$ be the objective function. Let   $w^*$ be any twice differentiable point of $L$ such that $L(w^*)=\frac{1}{2} \|\hat Y_n(w^*) -Y\|^2_F=0$. Then, if there exists $k\in \{1,\dots,H+1\}$ such that $N_{k} N_{k-1} >n \cdot \min(N_{k},N_{k+1},\dots,N_{H+1})$ where $N_0 =d$ and  $N_{H+1}=d'$ (i.e., overparametrization), there exists a zero eigenvalue of Hessian $\nabla^2 L(w^{*})$.    
\end{theorem}
\begin{proof}
Since $L(w)=\frac{1}{2} \|\hat Y_n(w) -Y\|^2_F=\frac{1}{2}\vect(\hat Y_n(w) -Y)^\top \vect(\hat Y_n(w) -Y)$, we have that $\nabla^2 L(w) = \nabla \vect(\hat Y_n(w)) \nabla \vect(\hat Y_n(w))^\top + \sum_{j=1}^{nd'}
\nabla^2 \vect(\hat Y_n(w))_j (\vect(\hat Y_n(w))_j - \vect(Y)_{j})$. Therefore,
$$
\nabla^2 L(w^{*}) = \nabla \vect(\hat Y_n(w^*)) \nabla \vect(\hat Y_n(w^*))^\top . 
$$
To obtain a formula of $\nabla \vect(\hat Y_n(w))$, we first write $\vect(\hat Y_n(w))$ recursively in terms of $\vect(h_k)$ as follows. For any $k \in\{2,\dots,H+1\}$,
$$
\vect(h_k) = W_{k} \sigma(h_{k-1}) = [I_n \otimes W_k] \vect(\sigma(h_{k-1}))=[I_n \otimes W_k] \sigma(\vect(h_{k-1})), 
$$
where the last equality follows the fact that $\sigma$ represents the elementwise application of nonlinear activation function (for an input of any  finite size). Thus, 
$$
\vect(\hat Y_n(w)) = \vect(h_{H+1})=[I_n \otimes W_{H+1}] \sigma(\vect(h_{H})),
$$ 
and we can expand $\vect(h_{H})$ recursively by the above equation of $\vect(h_k)$ for any $k \in\{2,\dots,H+1\}$. 

Moreover, for any $k \in\{2,\dots,H+1\}$,
$$  
\vect(h_k) = W_{k} \sigma(h_{k-1}) = [\sigma(h_{k-1})^\top \otimes I_{N_k}] \vect(W_k),
$$
and 
$$
\vect(h_1) =W_{1} X= [X^\top \otimes I_{N_1}] \vect(W_1). 
$$
From these, we obtain that $$
\nabla_k \vect(\hat Y_n(w))^{\top} = \tilde h_k [\sigma(h_{k-1})^\top \otimes I_{N_k}] \ \ \ \text{ for any } k \in \{2,\dots,H+1\},  
$$
and 
$$
\nabla_1 \vect(\hat Y_n(w))^{\top} = \tilde h_1 [X^{\top}\otimes I_{N_1}],
$$
where 
$$
\tilde h_k = \tilde h_{k+1} [I_{n} \otimes W_{k+1}]  \dot \sigma(h_k) \ \ \ \text{ for any } k \in \{1,\dots,H\},
$$
and 
$$
\tilde h_{H+1} = I_{nN_{H+1}}. 
$$
Therefore, for any $k \in \{2,\dots,H+1\}$,
\begin{align*}
\nabla_k \vect(\hat Y_n(w)) \nabla_k \vect(\hat Y_n(w))^{\top} 
=[\sigma(h_{k-1}) \otimes I_{N_k}]  \tilde h_k^{\top}  \tilde h_k [\sigma(h_{k-1})^\top \otimes I_{N_k}],
\end{align*}
which is a $N_{k} N_{k-1}$ by $N_{k} N_{k-1}$ matrix, and 
\begin{align*}
\nabla_1 \vect(\hat Y_n(w)) \nabla_1 \vect(\hat Y_n(w))^{\top} 
=[X\otimes I_{N_1}]  \tilde h_1^{\top}   \tilde h_1 [X^{\top}\otimes I_{N_1}],
\end{align*}
which is a $N_{1} d$ by $N_{1} d$ matrix.
Here, the rank of
$\nabla_k \vect(\hat Y_n(w)) \nabla_k \vect(\hat Y_n(w))^{\top} 
_k$ is at most $n \cdot \min(N_{k},N_{k+1},\dots,N_{H+1})$ for any $k \in \{1,\dots,H+1\}$
where $N_{H+1}=d'$, because the rank of the product of matrices is at most the minimum
of the ranks of matrices. This implies that for any $k \in \{1,\dots,H+1\}$,
 $\nabla_k \vect(\hat Y_n(w)) \nabla_k \vect(\hat Y_n(w))^{\top}$ has a zero eigenvalue if $N_{k} N_{k-1} >n \cdot \min(N_{k},N_{k+1},\dots,N_{H+1})$ where $N_0 =d$ and  $N_{H+1}=d'$. 
Therefore, the Hessian $\nabla^2 L(w^{*}) = \nabla \vect(\hat Y_n(w^*)) \nabla \vect(\hat Y_n(w^*))^\top \succeq 0$ has (at least) a zero eigenvalue if there exists $k\in \{1,\dots,H+1\}$ such that $N_{k} N_{k-1} >n \cdot \min(N_{k},N_{k+1},\dots,N_{H+1})$ where $N_0 =d$ and  $N_{H+1}=d'$.

\end{proof}

If $\sigma$ is additionally assumed to be ReLU, we can write that for any $k \in \{1,\dots,H+1\}$,
$$
\vect(\hat Y_n(w^{*}))=\nabla_k \vect(\hat Y_n(w^{*}))^{\top} \vect(W^{*}_k).
$$
The  directions  $\Delta_k \vect(w)=\vect(0,\dots,0,\Delta W_k,0,\dots,0)$  found in the  proof of Theorem \ref{thm:zero_eigenvalue} such that 
$$
\Delta_k \vect(w)^{\top} \nabla^2 L(w^{*}) \Delta_k \vect(w)=0
$$
corresponds to 
$$
\nabla_k \vect(\hat Y_n(w^{*}))^{\top} \Delta W_k =0.
$$
Note that $\nabla_k \vect(\hat Y_n(w))$ depends on $W_k$ only because of the dependence of the ReLU activations   on $W_k$. Thus, at a differentiable point $w^*$, there exists  a sufficiently small $\epsilon>0$ such that    
the model outputs on the training dataset $\hat Y_n(w)$ does not change in these directions   $\Delta_k \vect(w)=\vect(0,\dots,0,\Delta W_k,0,\dots,0)$ as 
\begin{align*}
\vect(\hat Y_n(w^{*}+ \epsilon\Delta_k \vect(w))) &=\nabla_k \vect(\hat Y_n(w^{*}))^{\top} (\vect(W^{*}_k)+ \epsilon\Delta W_k)
\\ & =\nabla_k \vect(\hat Y_n(w^{*})) ^{\top}\vect(W^{*}_k)
\\ &=\vect(\hat Y_n(w^{*})).
\end{align*}

\subsection{For linear networks SGD selects flat minima which  imply  robust
  optimization which provides maximization of margin}
\label{SGDLFlatRobust}

In the next subsections we extend previous results \cite{BEN:09} to
relate minimizers that corresponds to flatness in the minimum  to  robustness:
thus SGD performs  robust optimization of a certain type. In particular, we
will show that robustness to perturbations 
is related to regularization (see \cite{BEN:09,Xu2009}). We 
describe the separable and linear case.

The argument in \cite{Theory_IIb} shows why SGDL (stochastic gradient
descent Langevin) selects degenerate minimizers and, among those, the ones
corresponding to larger flat regions of the loss. In fact SDGL shows
concentration in probability -- {\it because of the
  high dimensionality} -- of its asymptotic distribution for minima that are the most robust to perturbations
of the weights.  The above fact suggests the following

\begin{conjecture}
  Under regularity assumptions on the landscape of the empirical
  minima, SGDL corresponds to the following robust optimization
\begin{equation}
\min_{w} \max_{(\delta_1,\cdots,\delta_n)} \frac{1}{n}\sum_{i=1}^n V(y_i,
f_{w+\delta_i w}(\mathbf{x_i})).
\end{equation}
\label{FlatnessSGDL-1}
\end{conjecture}

It is especially important to note that {\it SGDL -- and approximately
  SGD -- maximize volume and ``flatness'' of the loss in weight
  space}.

\subsubsection{SGD  converges
  to a large-margin classifier for linear functions}

{\it Assume} that the conjecture \ref{FlatnessSGDL-1} holds and
consider a binary classification problem. In linear classification we try to separate the two classes of a
binary classification problem by a hyperplane $\mathcal{H}=\{x :
w^\intercal x +b =0 \}$, where $w \in \mathbb{R}^d$. There is a
corresponding decision rule of the form $y =sign (w^\intercal x +b)$
Let us now assume a  {\it separability condition} on the data set: the decision rule
makes no error on the data set. This corresponds to the following set
of inequalities

\begin{equation}
y_i(w^\intercal x_i + b) > 0, \quad i=1, \cdots,n
\label{linearhyperplane}
\end{equation}

The separability conditions implies that  the inequalities are feasible.
Then (changing slightly the notation for the classification case)  the
robustness formulation of Equation \ref{FlatnessSGDL-1} is equivalent
to robustness of the classifier wrt perturbations $\delta w$ in the weights which we assume here to be
such that $||\delta w||_2 \leq \rho ||w||_2$ with $\rho \geq
0$. In other words, we require that $\forall \delta w$ such that $\|\delta w\|_2 \leq \rho \|w\|_2$:

\begin{equation}
 y_i ((w + \delta w)^\intercal x_i + b ) \geq 0, \quad i=1,\ldots,n
\end{equation}

Further assume that $\|x_i\|_2 \approx 1$, then for $i=1,\ldots,n$, if we let $\delta w=-\rho
y_ix_i\|w\|_2/\|x_i\|_2$,
\[
  y_i (w^\intercal x_i + b) \geq -y_i \delta w^\intercal x_i = \rho\|w\|_2 \|x_i\|_2 \approx
  \rho\|w\|_2
\]

We obtain a  \emph{robust counterpart} of
Equation~\eqref{linearhyperplane}  

\begin{equation}
y_i(w^\intercal x_i + b) \geq  \rho ||w||_2, \quad i=1, \cdots,n
\label{RobustCounterpart}
\end{equation}

Maximizing $\rho$ -- the margin -- subject to the constraints  Equation
\eqref{RobustCounterpart} leads to minimizing $w$,  because we can
always enforce $\rho ||w||_2=1$ and thus to

\begin{equation}
\min_{w,b} \{||w||_2 : y_i(w^\intercal x_i + b) \geq 1 \quad  i=1,
\cdots,n \}
\label{RobustOptimization}
\end{equation}

\noindent

We have proven the following result (a variation of 
\cite{BEN:09}, page 302)

\begin{lemma}
\label{robustness-lemma}
Maximizing flatness (as SGD does \cite{Theory_IIb}) is
equivalent to maximizing robustness and thus classification margin in
binary classification of separable data.
\end{lemma}

{\bf Remarks}

\begin{enumerate}

\item 
Notice that the classification problem is  similar to using the
loss function $V(y, f(x))=\log(1+ e^{-yf(x)})$ which penalizes errors
but is otherwise very small in the zero classification error
case. \cite{Theory_IIb} implies that SGD maximizes flatness
at the minimum, that is SGD maximizes $\delta w$.

\item The same Equation \ref{RobustOptimization} follows if the
maximization is with respect to spherical perturbations of radius
$\rho$ around each data point $x_i$. In either case the resulting
optimization problems is equivalent to hard margin SVMs.

\item If we
  start from the less natural assumption that $||\delta w||_\infty \leq \rho ||w||_2$ with $\rho \geq
0$ and assume that each of the $d$ components $|(x_i)_j| \approx 1, \quad
\forall
i$, then the {\it robust counterpart} of Equation \ref{linearhyperplane} is,
since $(\delta w)^\intercal x \leq ||(\delta
w)^\intercal||_{\ell_{\infty}} ||x||_{\ell_1}$,

\begin{equation}
y_i(w^\intercal x_i + b) \geq (\delta w)^\intercal \sup \delta w = \rho ||w||_2, \quad i=1, \cdots,n
\label{RobustCounterpart-2}
\end{equation}.

\item In the non-separable case, the hinge loss
  $V(y, f(x))= [1-yf(x)]_+$, with $y$ binary (used by \cite{Xu2009})
  -- which is similar to the cross-entropy loss
  $V(y, f(x))=\log(1+ e^{-yf(x)})$ -- leads to the following robust
minimization problem

\begin{equation}
\min_{w,b}  \frac{1}{n} \sum _{i=1}^n [1-y_i(w^\intercal x_i + b) +\rho ||w||_2]_+.
\label{HingeRobustOptimization}
\end{equation}

Note that the robust version of this worst-case loss minimization is
not the same as in classical SVM because the regularization term is
inside the hinge loss. Note also that standard regularization is an
upper bound for the robust minimum since

\begin{equation}
\min_{w,b}  \frac{1}{n} \sum _{i=1}^n [1-y_i(w^\intercal x_i + b) +\rho
||w||_2]_+ \leq \min_{w,b} \frac{1}{n} \sum _{i=1}^n [1-y_i(w^\intercal x_i +
b)]_+ + \rho ||w||_2].
\label{HingeRobustOptimizationBound}
\end{equation}

In the case of the square loss, robust optimization gives with
$||\delta w||_2 \leq \rho ||w||_2$


\begin{equation}
\min_{w,b}  \frac{1}{n} \sum _{i=1}^n [(y_i - w^\intercal x_i)^2 + \rho^2 ||w||_2]_+.
\label{SquareRobustOptimization}
\end{equation}

The summary here is that depending on the loss function and on the
uncertainity set allowed for the perturbations $\delta w$ one obtains
a variety of robust optimization problems. In general they are not
identical to standard regularization but usually they contain a
regularization term. Notice that a variety of regularization norms
(for instance $\ell_1$ and $\ell_2$) guarantee $CV_{loo}$ stability
and therefore generalization. The conclusion is that {\it in the case
  of  linear networks and linearly separable data, SGD provides a
solution closely related to hinge-loss SVM.}

\end{enumerate}

\subsubsection{Robustness wrt Weights and Robustness wrt Data}
\label{RobPert}
A natural intuition
is that several forms of stability are closely related. In particular,
{\it perturbations of the weights follow from perturbations of the
  data points in the training set} because the function $f(x)$ resulting from the training
is parametrized by the weights that depend on the data points $S_n$:
$f(x) = f(x; w (z_1, \cdots, z_n))$, with $z= (x,y)$\footnote{In a
  differentiable situation,  one would write $df = \frac{df}{dw}
  \frac{dw}{dS} dS$ and $dw= \frac{dw}{dS} dS$. The latter equation
  would show that perturbations in the weights depend on perturbations of
  the training set $S$.} . Thus {\it flat regions in weight space of the global minimum of the
  empirical loss indicate stability with respect to the weights; the
  latter in turn indicates  stability with respect to training examples}.

\subsection{Polynomial Networks}
\label{SectionPolynomials}

There are various reasons why it is interesting to approximate
nonlinear activation functions with univariate polynomials thus
leading to deep polynomial networks. For this paper polynomial
approximations of deep networks justify the use of Bezout theorem to
characterize minima of the loss function. In addition, it is
worthwhile to remark that the good empirical approximation by
polynomial networks of the main properties of ReLU networks implies
that specific properties of the ReLU activation function such as its
discontinuous derivative, and the property of {\it non-negative
  homogeneity}, e.g. $\sigma(\alpha z) = \alpha \sigma(z)$ do not play
any significant role. Furthermore, the characterization of certain
properties of deep networks becomes easier by thinking of them as
polynomial and thus analytic functions.

A generic polynomial $P_k(x)$ of degree $k$ in $d$ variables is in the
linear space $\mathcal{P}_k = \cup_{s=0}^k H_s$ composed by the union
of homogeneous polynomials $\mathcal{H}_s$, each of degree $s$.
$\mathcal{H}_k$ is of dimensionality
$r= dim \mathcal{H}_k= \binom{d-1+k}{k}$: the latter is the number of
monomials and thus the number of coefficients of the polynomials in
$\mathcal{H}_k$. A generic polynomial in $\mathcal{P}_k$ can always be
written (see Proposition 3.5 in \cite{Mhaskar1993}) as

\begin{equation}
P_k(x) = \sum _{i=1}^r p_i (\scal{w_i}{x}).
\label{Linea}
\end{equation}

\noindent where $p_i$ is a  univariate polynomial of degree at most
$k$ and $\scal{w_i}{x})$ is a scalar product. In fact

\begin{equation}
P(x) = \sum _j a_j x^j = \sum _{k=1}^N c_k ((w_k, x)+b_k)^n = \sum_{k=1}^N c_k \sum_{|j|=n} \binom{n}{j}
u^j_k y^j
\label{HMhaskar}
\end{equation}
\noindent with $y=(x,1)$ and $u_k=(w_k, b_k)$ so that $a_j= \binom{n}{j}
\sum_{k=1}^N c_k u_k^j.$

Notice that the representation of a polynomial in $d$ variables with
total degree $\leq k$ (so dimension $N$) with monomial or any other
fixed basis involves exactly $N$ parameters to optimize. In the
representation that is produced by a singe hidden layer in a network
as $\sum_k a_k (w_k\cdot x+b_k)^k$, there are $(d+2) N$ weights.


\subsubsection{Polynomial approximation of deep networks}
\label{simulapproxsect}
First, we discuss the univariate case. 
Let $C^*$ be the class of all continuous, univariate, $2\pi$-periodic functions on $\mathbb{R}$, equipped with the supremum norm $\|\cdot\|^*$, 
and $\mathbb{H}_n$ be the class of all trigonometric polynomials of order $<n$. 
For $f\in C^*$, let
$$
E_n^*(f)=\min_{T\in \mathbb{H}_n}\|f-T\|^*.
$$
The following simple theorem was proved by Czipser and Freud \cite[Theorem~1]{czipser1957approximation}:

\begin{theorem}\label{czipserfreudtheo}
Let $r\ge1$ be an integer, $f\in C^*$ be $r$ times continuously differentiable. 
For integer $n\ge 1$, and $\epsilon>0$, if $T$ is a trigonometric polynomial of order $<n$ such that
$$
\|f-T\|^* \le \epsilon,
$$
then
$$
\|f^{(r)}-T^{(r)}\|^* \le  3\cdot 2^r n^r\epsilon+4E_n^*(f^{(r)}).
$$
where the minimum is taken over all trigonometric polynomials $P$ of order $<n$. 
\end{theorem}

To discuss the aperiodic case, let $C$ be the class of all continuous functions on $[-1,1]$, equipped with the supremum norm $\|\cdot\|$, 
and $\Pi_n$ be the class of all algebraic polynomials of degree $<n$. For $f\in C$, we define
$$
E_n(f)=\min_{P\in\Pi_n}\|f-P\|.
$$
Writing $\phi(x)=\sqrt{1-x^2}$, it is easy to deduce the following corollary of Theorem~\ref{czipserfreudtheo} using the standard substitution $x=\cos \theta$:
\begin{corollary}\label{algczipserfreudcor}
Let $f\in C$ be continuously differentiable, $n\ge 1$ and $P\in \Pi_n$. If
\be\label{polyapprox}
\|f-P\|\le \epsilon,
\ee
then
\be\label{derpolyapprox}
\|(f'-P')\phi\| \le 3\cdot 2^r n^r\epsilon+4E_{n-1}(\phi f').
\ee
\end{corollary}

Thus, if $f$ is a smoothened version of the relu function, and $P$ is an approximation to $f$ satisfying (\ref{polyapprox}), then $\nabla_{\mathbf{w},b} f(\langle \mathbf{w}, \mathbf{x}\rangle +b)$ can be approximated by $\nabla_{\mathbf{w},b} P(\langle \mathbf{w}, \mathbf{x}\rangle +b)$ in the weighted uniform norm as in (\ref{derpolyapprox}). 

The extension of Corollary~\ref{algczipserfreudcor} to higher derivatives is not so straightforward.  For $\gamma>0$, we define for $f\in C$,
$$
\|f\|_\gamma =\|f\|+\sup_{j\ge 0} 2^{j\gamma}E_{2^j}(f),
$$
and let $W_\gamma$ be the space of all $f\in C$ for which $\|f\|_\gamma>0$. 
Thus, if $r\ge 1$ is an integer, and $f$ is $r$ times continuously differentiable, then $f\in W_r$. 
However, $W_r$ is the not the same as the space of all $r$ times continuously differentiable functions on $[-1,1]$. 
A complete characterization of the spaces $W_\gamma$ is given in \cite[Chapter~8, Section~7]{devlorbk}.
The following theorem can be proved using the Bernstein inequality \cite[Chapter~8, Theorem~7.6]{devlorbk}.

\begin{theorem}\label{algsimapproxtheo}
Let $r\ge 1$ be an integer, $\gamma>r$, $f\in W_\gamma$. Let $n\ge 1$ be an integer, $P\in\Pi_n$, and (\ref{polyapprox}) hold for some $\epsilon>0$. Then
\be\label{highderapprox}
\|(f^{(r)}-P^{(r)})\phi^r\| \le cn^r\{\epsilon + n^{-\gamma}\|f\|_\gamma\},
\ee
where $c>0$ is a constant depending only on $r$ and $\gamma$.
\end{theorem}

Theorem~\ref{algsimapproxtheo} can be extended to the multi-variate
setting using tensor products. Finally, we note that theorems
analogous to Theorem~\ref{algsimapproxtheo} are given in
\cite{hahm2009simultaneous}.  However, these theorems merely assert
the existence of neural networks evaluating a sigmoidal $\tanh$
activation function so that the derivatives of these networks
approximate the derivatives of the function.

\subsubsection{Deep polynomial networks behave as deep networks}

Nonlinear activations in a deep networks can be approximated up to an
arbitrary accuracy over a bounded interval by a univariate polynomial
of appropriate degree (see \cite{Theory_II}) which approximates well
also the derivative of the function (see SI
\ref{simulapproxsect}). Since so much is known about polynomials,
several interesting results follow from these approximation properties
such as a characterization of when deep convolutional networks can be
exponential better than shallow networks in terms of representational
power (\cite{Theory_I}).  We show here empirical evidence that the
puzzling generalization properties of deep convolutional networks hold
for networks with the same architecture in which all the ReLU's have
been replaced by a univariate polynomial approximation (of degree $10$
in our simulations).  The resulting 5-layers networks are obviously
polynomial networks, that compute a polynomial in the input $x$ -- and
{\it that are also polynomials in the weight parameters}. Figure
\ref{fig:relu_vs_polynomial} demonstrates that polynomial and ReLU
networks with the same architecture have very similar performance in
terms of training and testing error.  The same puzzling lack of overfitting
in the overparametrized case can be seen in Figure
\ref{GreatPlotPol}. Together Figures
\ref{TwoRegimesPol},\ref{GreatPlotPol}
include all the puzzles discussed in the main text and described by
\cite{DBLP:journals/corr/ZhangBHRV16}.

\begin{figure}
  \centering  
\makebox[0pt]{  \includegraphics[width=1.2\textwidth]{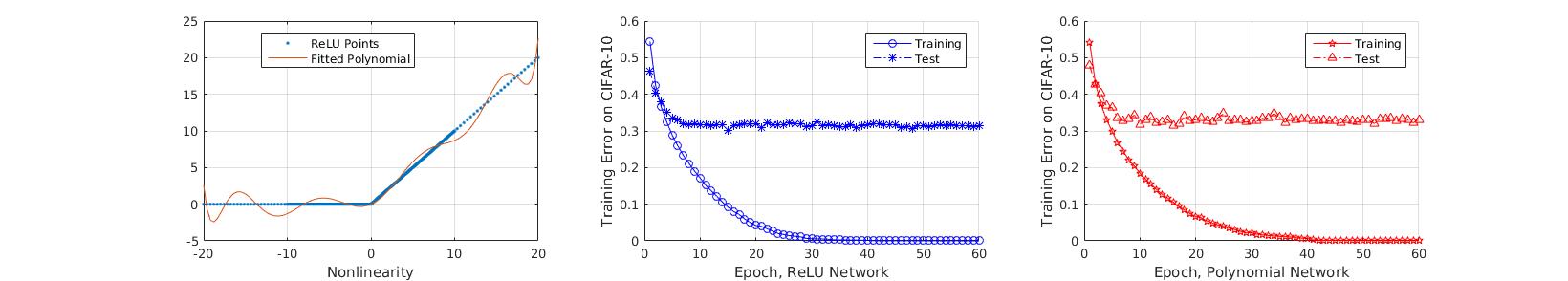}  }
\caption{\it A standard convolutional deep network is converted into a
  polynomial function by replacing each of the REL units with a
  univariate polynomial approximation of the ReLU function. As long as
  the nonlinearity approximates the ReLU function well (especially
  near 0), the ``polynomial network'' performs quantitatively
  similarly to the corresponding ReLU net. The polynomial shown in the
  inset on the left is of degree $10$.}
\label{fig:relu_vs_polynomial}   
\end{figure}

\begin{figure*}[h!]\centering
\includegraphics[width=1.0\textwidth]{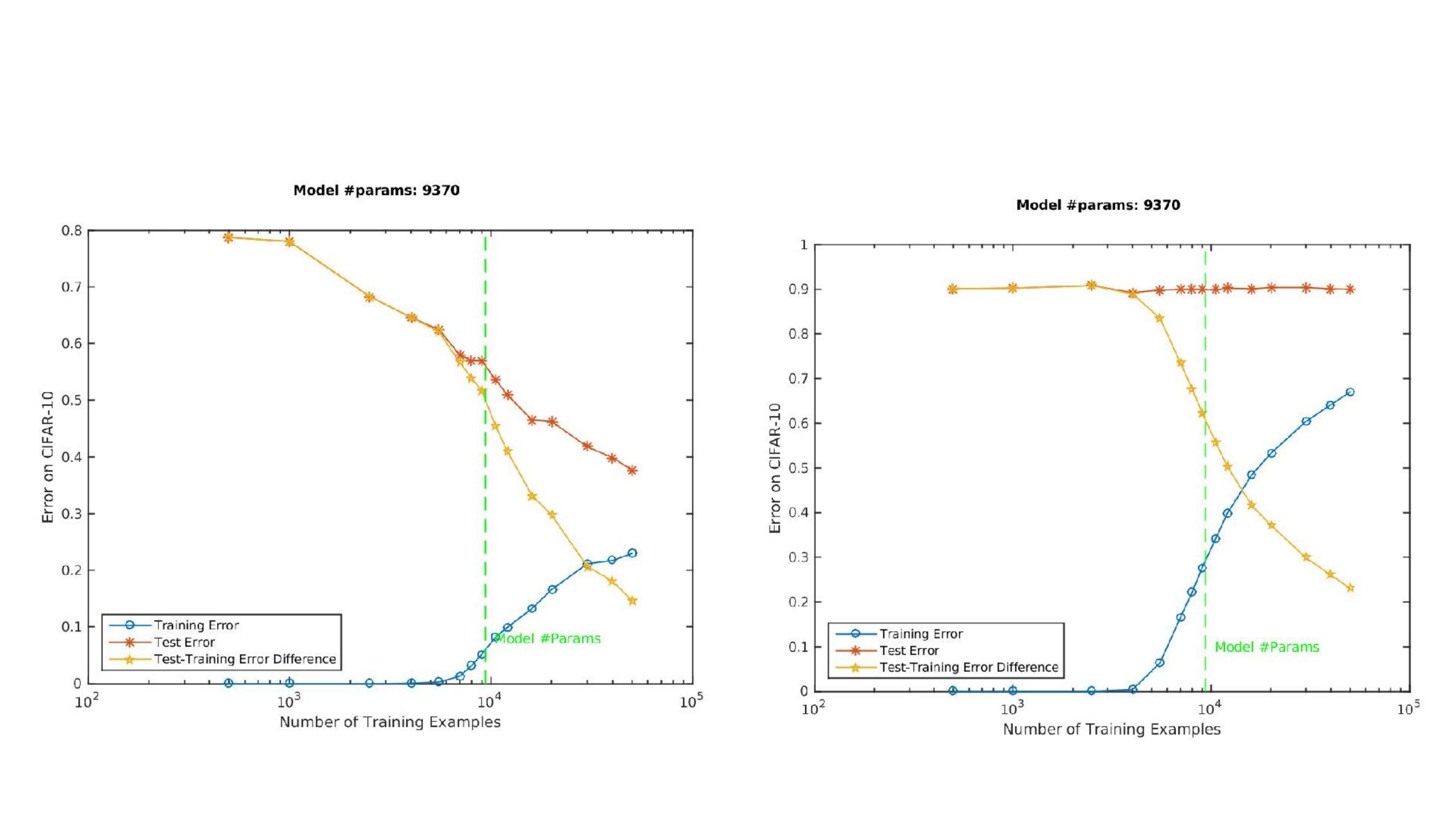}
\caption{\it The figure (left) shows the behavior of a polynomial deep network trained on
  subsets of the CIFAR database. The figure on the right shows the
  same network trained on
  subsets of the CIFAR database in which the labels have been randomly
  scrambled.  The network is a 5-layer all
  convolutional network (i.e., no pooling) with 16 channels per hidden
  layer , resulting in only $W \approx 10000$ weights instead of the
  typical $300,000$. Neither data augmentation nor regularization is performed.} 
\label{TwoRegimesPol}
\end{figure*}

As we discussed in the main text, Figure \ref{GreatPlot} shows that
the increase in the number of parameters does not induce overfitting,
at least in terms of classification error.  We show in Figure
\ref{GreatPlotPol} that the same behavior is shown for polynomial
networks.  As shown in Figure \ref{weightsNorm_layer_6}, the norm of
the weights increases during training until an asymptotic value is
reached,

\begin{figure*}[h!]\centering
\includegraphics[width=1.0\textwidth]{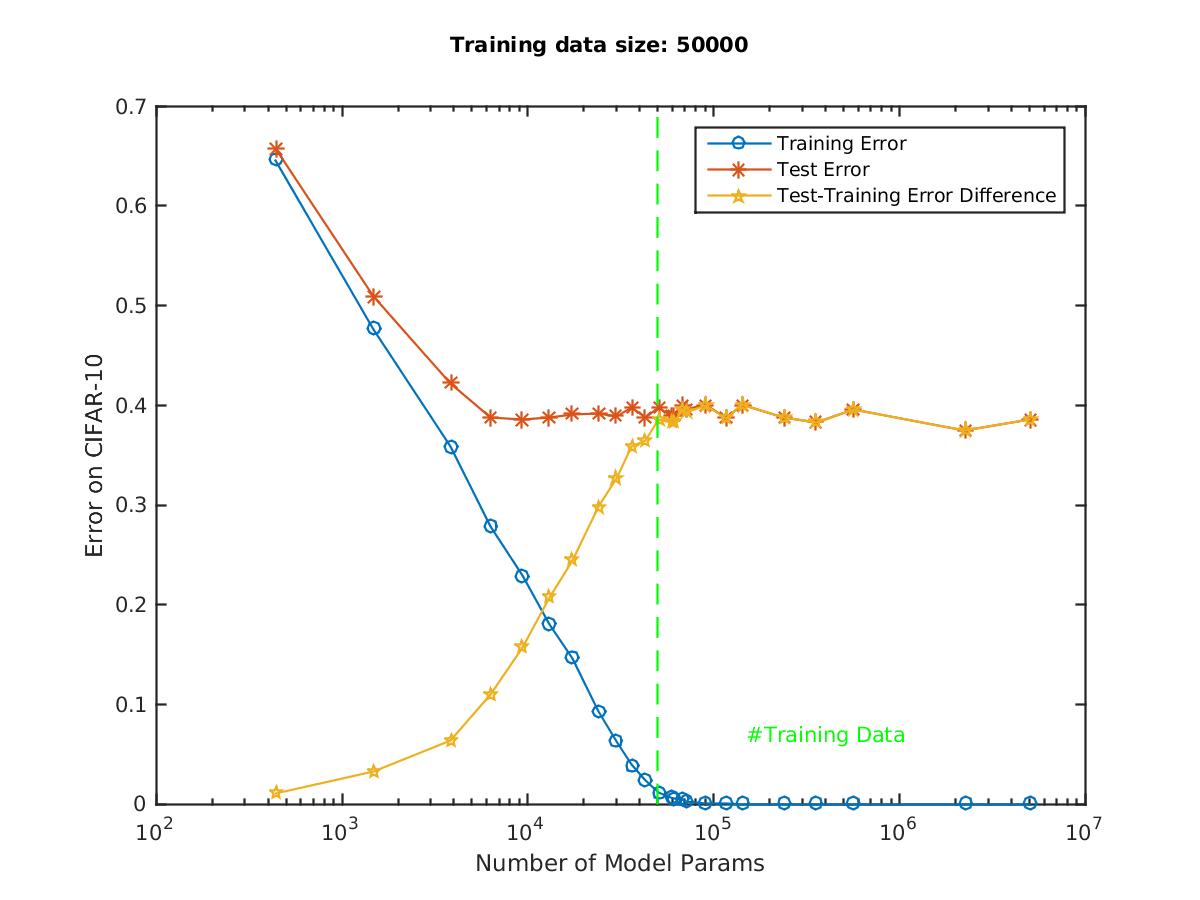}
\caption{\it The previous figures show dependence on $N$ -- number of
  training examples -- for a fixed architecture with $W$ parameters.
  This figure shows dependence on $W$ for a fixed training set with
  $N$ examples. The network is again a 5-layer all convolutional
  polynomial network. All hidden layers have the same number of
  channels. Neither data augmentation nor regularization is
  performed. The classical theory explains the generalization behavior
  on the left; the challenge is to explain the lack of overfitting for
  $W>n$. As shown here, there is zero
  error for $W \ge n$.}
\label{GreatPlotPol}
\end{figure*}


\begin{figure*}[h!]\centering
\includegraphics[width=1.0\textwidth]{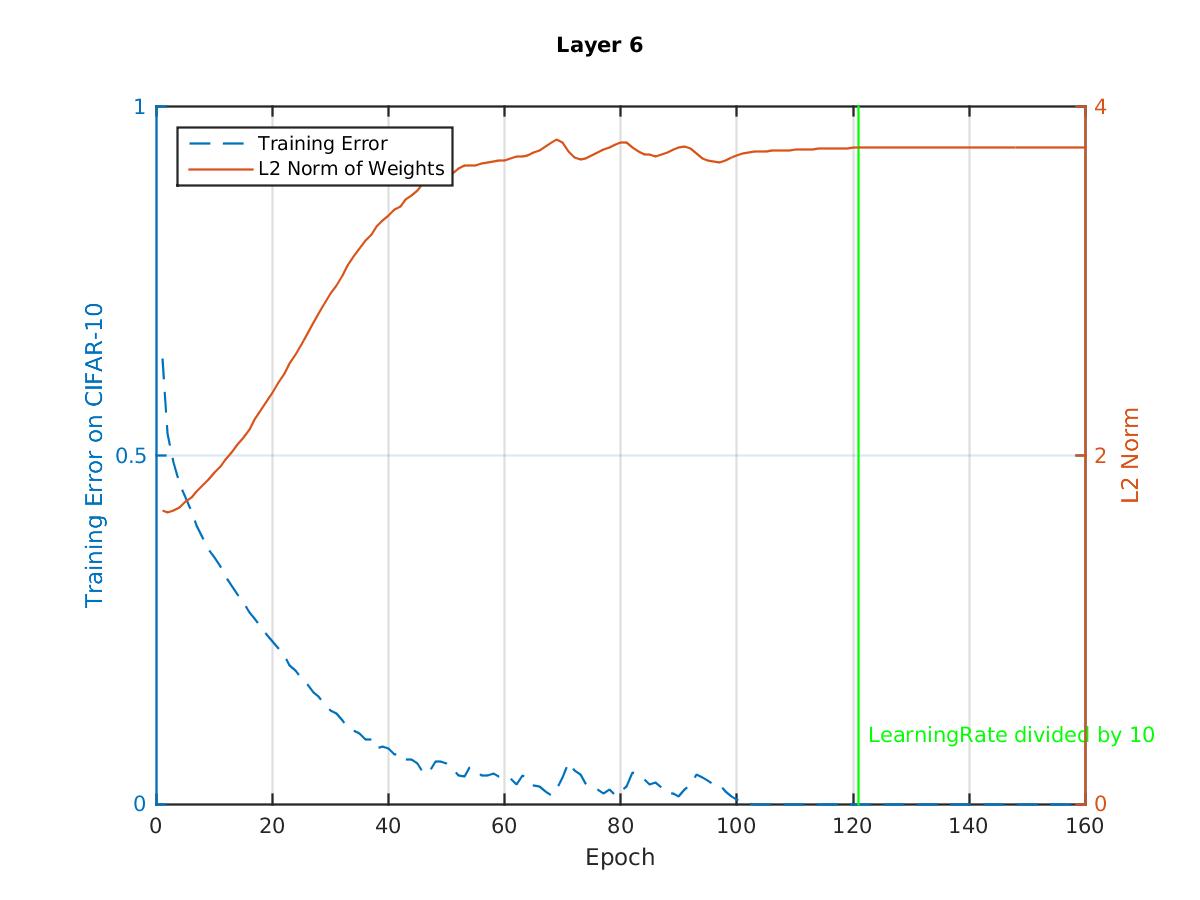}
\caption{\it The norm of the weights increases with training epochs of
  SGD until a certain level after which it does not change
  anymore. The behaviour shown here is typical for all layers for CIFAR-10.}
\label{weightsNorm_layer_6}
\end{figure*}

The same very high degree polynomial represented by a 5-layers network
with univariate polynomial activation of its hidden units can be, in
principle, parametrized in a standard way in terms of all the relevant
monomials, with each parameter corresponding to the coefficient of one
of the monomials\ref{SectionPolynomials}. Such a parametrization
corresponds to a one-hidden layer network where the weights of the
first layer are fixed (for instance equal to 1), each hidden unit
computes one of the monomial from the $d$ inputs and the learnable
weights are from the hidden units to the output.

It is well known that the coefficients of the polynomial in the standard
parametrization can be learned by using gradient descent or stochastic
gradient descent. This corresponds to optimizing  a linear network with square
loss which is equivalent to minimizing  a possibly degenerate
quadratic loss function.

\subsection{Additional experiments}

Figure \ref{Brando3} shows the testing error for an overparametrized
linear network optimized under the square loss. The Figure is related
to Figure \ref{Corrige:GreatPlot} in the main text.

\begin{figure*}[h!]\centering
\includegraphics[width=0.6\textwidth]{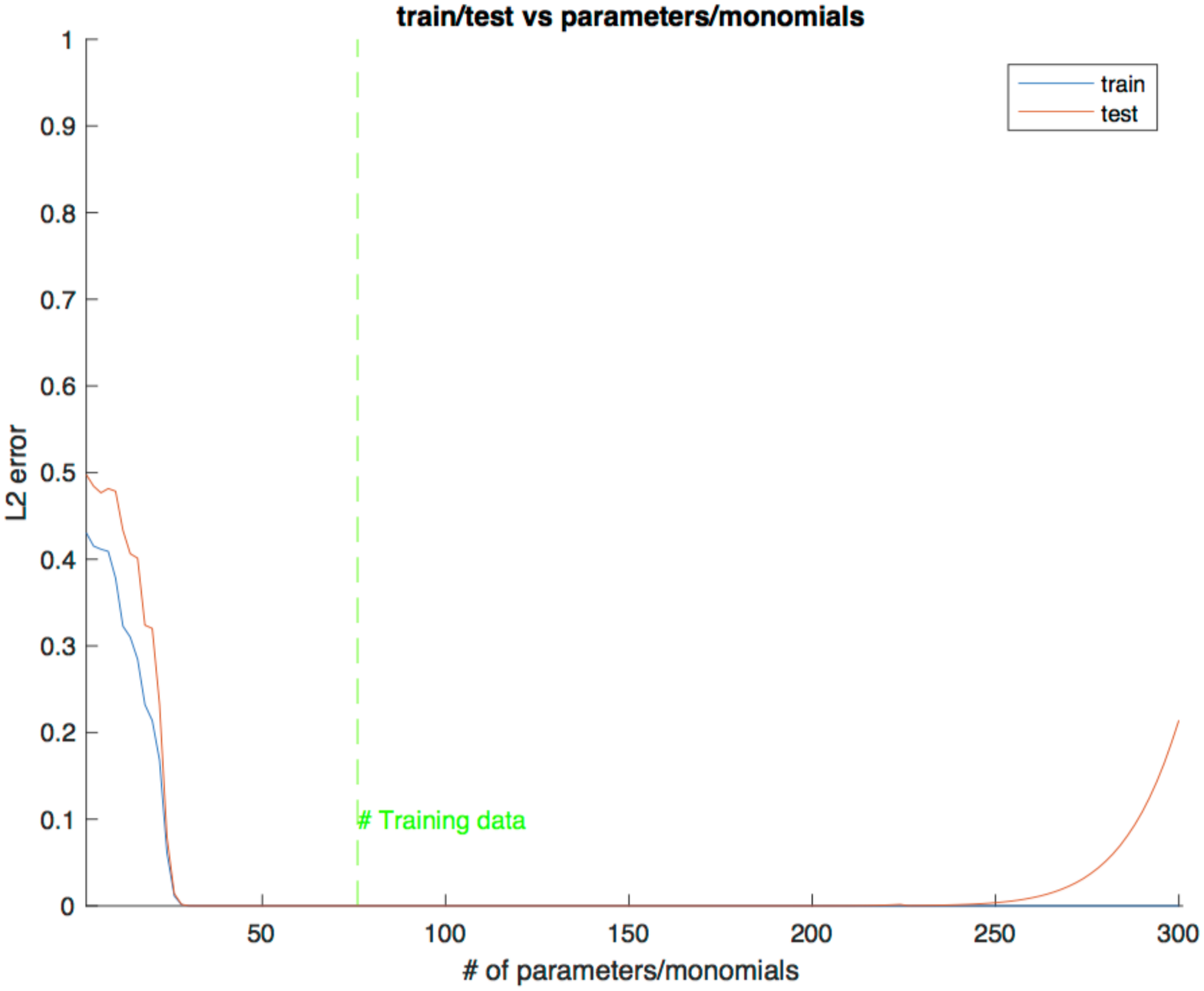}
\caption{\it Training and testing with the square loss for a linear
  network in the feature space (i.e. $y = W\phi(X)$) with a degenerate
  Hessian of the type of Figure \ref{degeneratesquareloss}. The
  feature matrix is a polynomial with increasing degree, from 1 to
  300.  The square loss is plotted vs the number of monomials, that is
  the number of parameters.  The target function is a sine function
  $f(x) = sin(2 \pi f x ) $ with frequency $f=4$ on the interval
  $[-1,1]$.  The number of training points where $76$ and the number
  of test points were $600$.  The solution to the over-parametrized
  system was the minimum norm solution.  More points where sampled at
  the edges of the interval $[-1,1]$ (i.e. using Chebyshev nodes) to
  avoid exaggerated numerical errors.  The figure shows how eventually
  the minimum norm solution overfits.}
\label{Brando3}
\end{figure*}


Figure \ref{CIFARclass} shows the behavior of the loss in CIFAR in the
absense of perturbations. This should be compared with Figure
\ref{Brando1} which shows the case of an overparametrized linear
network under quadratic loss corresponding to the multidimensional
equivalent of the degenerate situation of Figure
\ref{degeneratesquareloss}. The nondegenerate, convex case is shown in
Figure \ref{Brando2}. Figure

\begin{figure*}[h!]\centering
\includegraphics[width=1.0\textwidth]{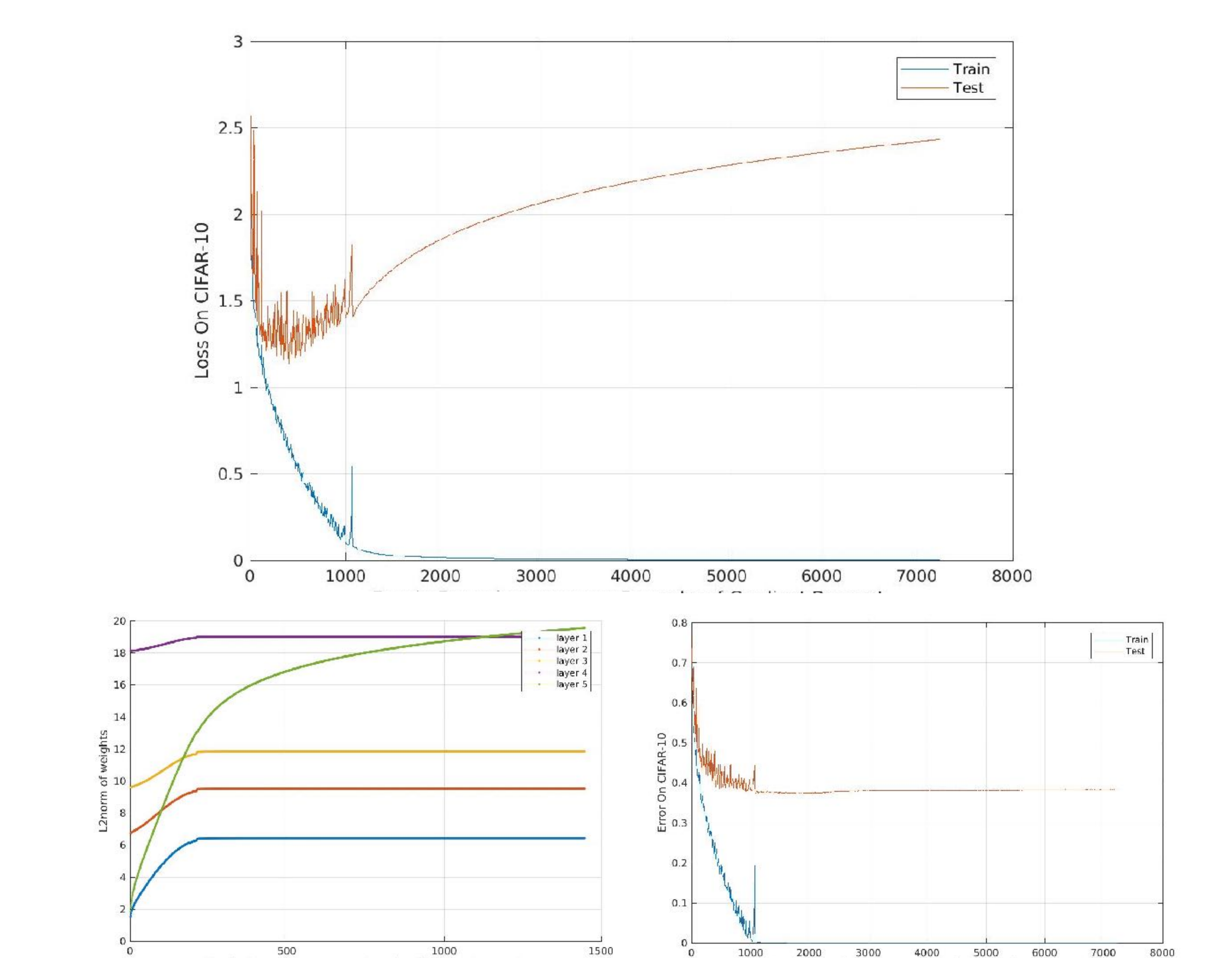}
\caption{\it Same as Figure 4 but without perturbations of weights.
  Notice that there is some overfitting in terms of the testing
  loss. Classification however is robust to this overfitting (see
  text).}
\label{CIFARclass}
\end{figure*}

\begin{figure*}[h!]\centering
\includegraphics[width=1.0\textwidth]{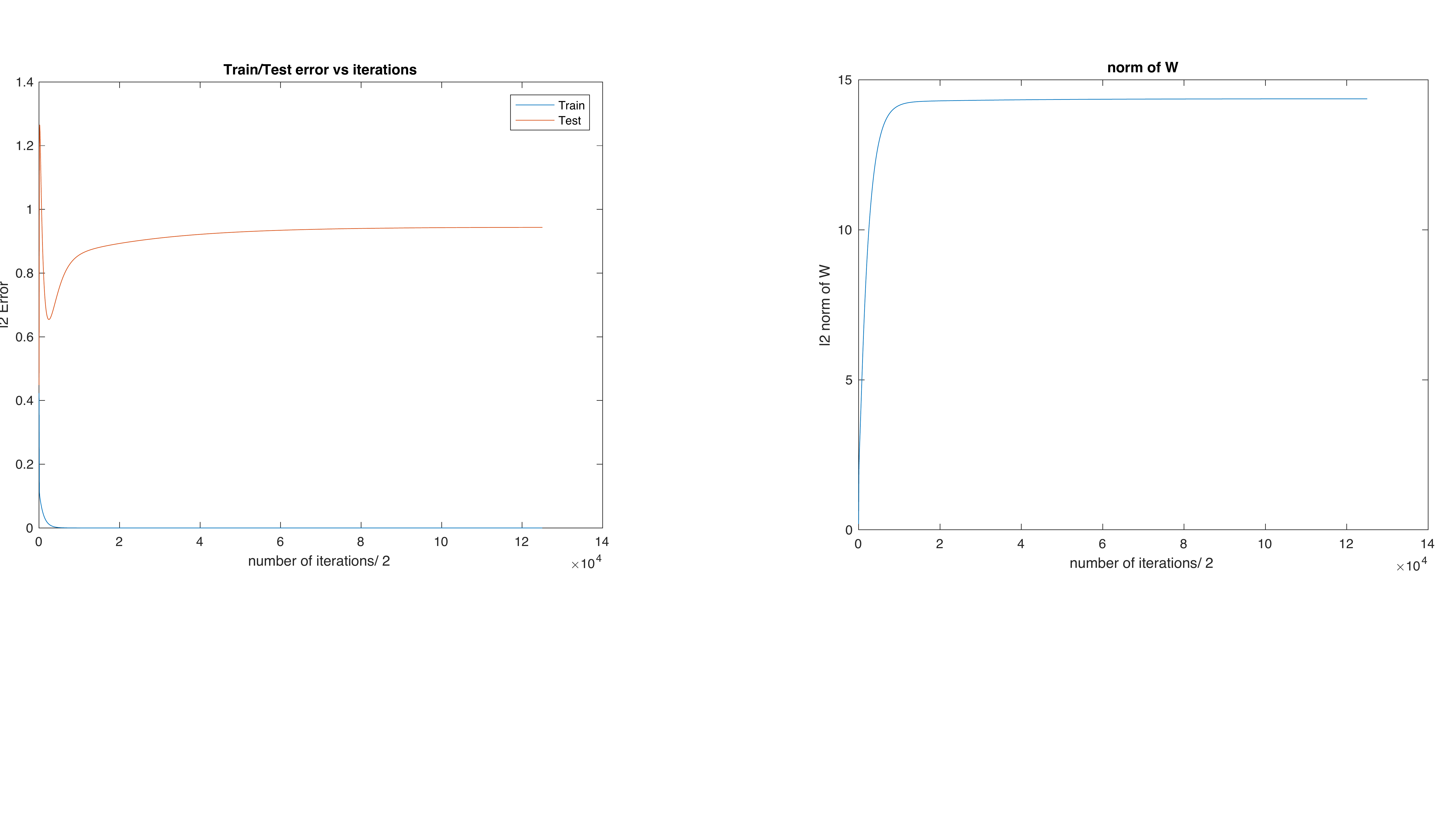}
\caption{\it Training and testing with the square loss for a linear
  network in the feature space (i.e. $y=W\Phi(X)$) with a degenerate
  Hessian of the type of Figure \ref{degeneratesquareloss}.  The
  feature matrix $\phi(X)$ is a polynomial with degree 30.  The target
  function is a sine function $f(x) = sin(2 \pi f x) $ with frequency
  $f=4$ on the interval $[-1,1]$.  The number of training point are
  $9$ while the number of test points are $100$.  The training was
  done with full gradient descent with step size $0.2$ for $250,000$
  iterations.  The weights were not perturbed in this experiment.  The
  $L_2$ norm of the weights is shown on the right.  Note that training
  was repeated 30 times and what is reported in the figure is the
  average train and test error as well as average norm of the weights
  over the 30 repetitions. There is overfitting in the test error.}
\label{Brando1}
\end{figure*}

\begin{figure*}[h!]\centering
\includegraphics[width=1.0\textwidth]{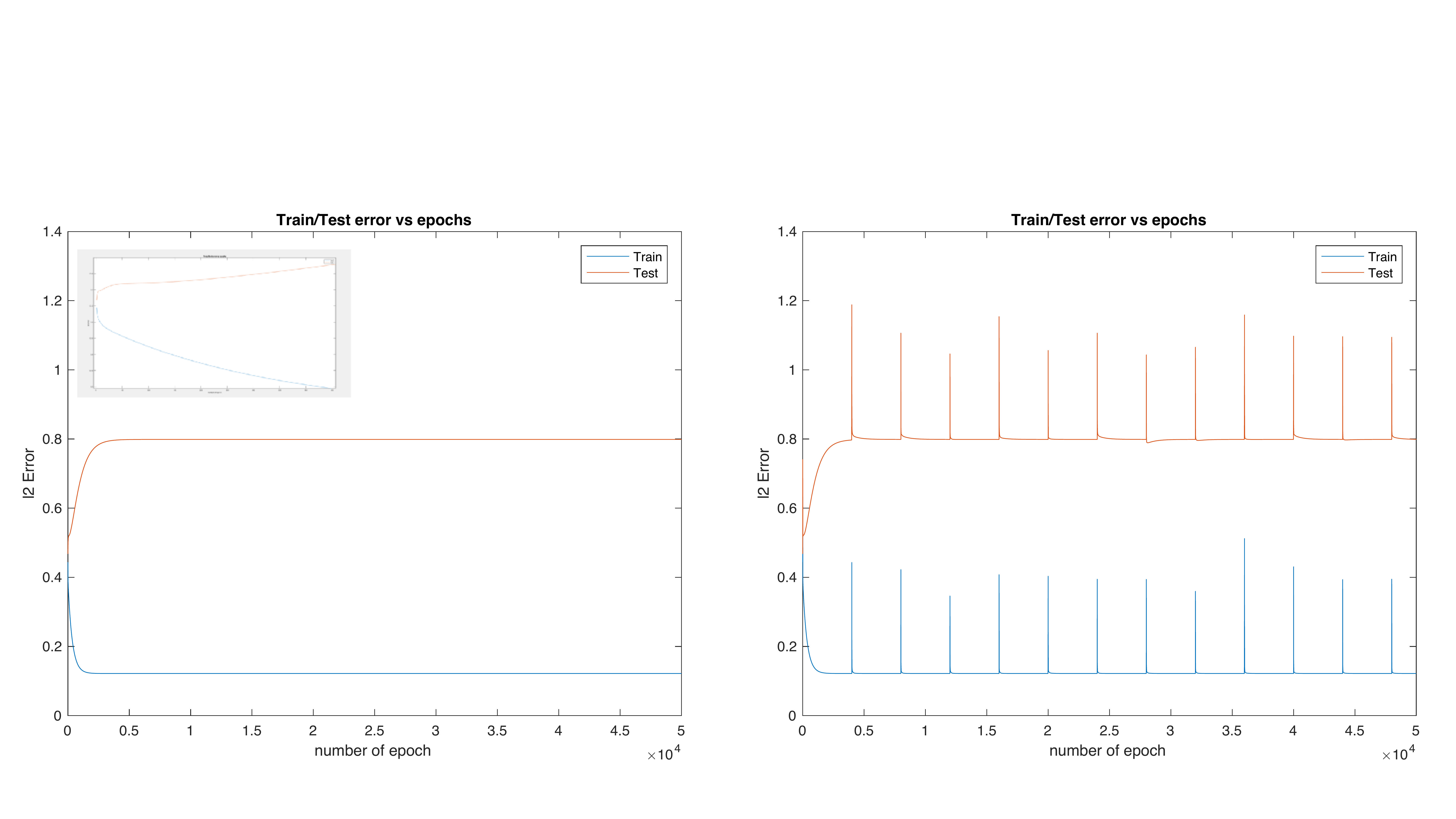}
\caption{\it The graph on the left shows training and testing loss for
  a linear network in the feature space (i.e. $y=W\Phi(X)$) in the
  nondegenerate quadratic convex case.  The feature matrix $\phi(X)$
  is a polynomial with degree 4.  The target function is a sine
  function $f(x) = sin(2 \pi f x) $ with frequency $f=4$ on the
  interval $[-1,1]$.  The number of training point are $9$ while the
  number of test points are $100$.  The training was done with full
  gradient descent with step size $0.2$ for $250,000$ iterations. The
  inset zooms in on plot showing the absense of overfitting. In the
  plot on the right, weights were perturbed every $4000$ iterations
  and then gradient descent was allowed to converge to zero training
  error after each perturbation.  The weights were perturbed by adding
  Gaussian noise with mean $0$ and standard deviation $0.6$.  The plot
  on the left had no perturbation.  The $L_2$ norm of the weights is
  shown on the right.  Note that training was repeated 30 times and
  what is reported in the figure is the average train and test error
  as well as average norm of the weights over the 30 repetitions.}
\label{Brando2}
\end{figure*}

As shown in Figure \ref{SquareLossPol} and Figure \ref{SquareClass}, qualitative
  properties of deep learning networks under the crossentropy loss
  function seem to hold, as expected, under the quadratic loss. 
\begin{figure*}[h!]\centering
\includegraphics[width=1.0\textwidth]{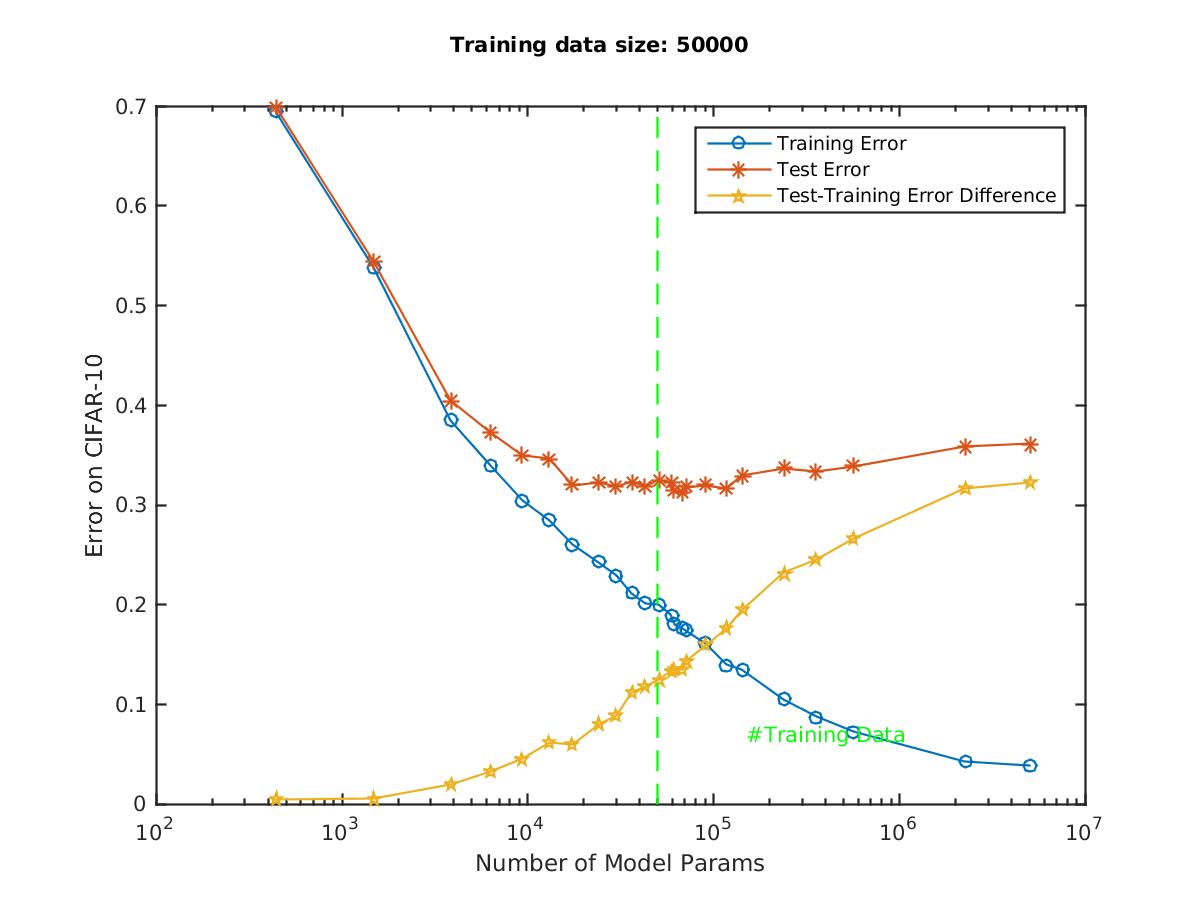}
\caption{\it The figure shows the behavior of a deep polynomial
  network trained on the CIFAR database, using the square loss. To be
  compared with Figure \ref{GreatPlot}}.
\label{SquareLossPol}
\end{figure*}

\begin{figure*}[h!]\centering
\includegraphics[width=1.0\textwidth]{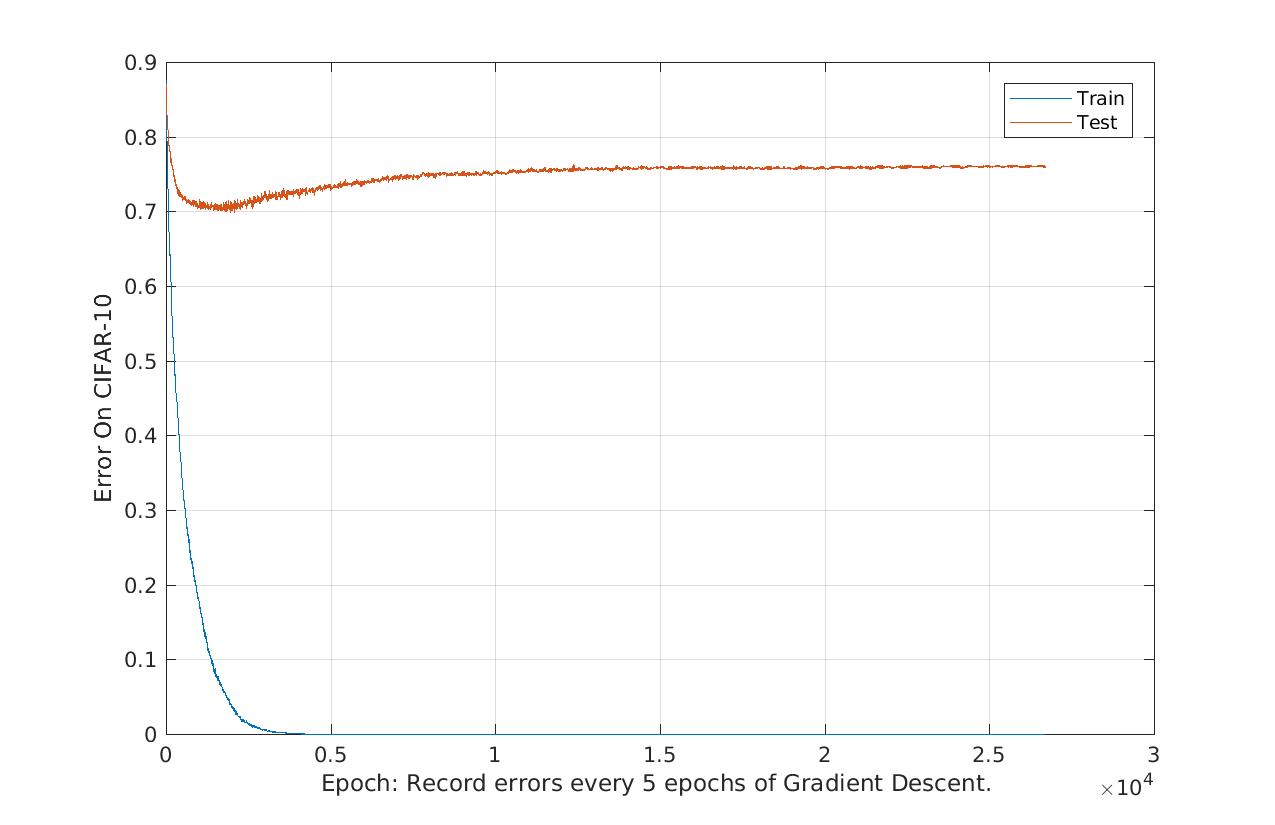}
\caption{\it Classification  error  on CIFAR obtained with GD
  optimizing the square loss risk. The training set has $2000$
  examples and the network has $188810$ parameters. Overfitting
  appears here for the classification error.}
\label{SquareClass}
\end{figure*}

\end{document}